%% file: main.tex
\def\arxivversion{1}
\documentclass{article}

\makeatletter
\def\input@path{{./}}
\makeatother
\usepackage{iclr2027_conference,times}
\usepackage[T1]{fontenc}
\usepackage{microtype}
\usepackage{amsmath,amssymb,amsthm}
\usepackage{booktabs}
\usepackage{graphicx}
\usepackage{tabularx}
\usepackage{xcolor}
\usepackage{xurl}
\usepackage{xspace}
\usepackage[hidelinks]{hyperref}
\usepackage[nameinlink,noabbrev]{cleveref}
\usepackage[section]{placeins}

\AddToHook{env/table/begin}{\setlength{\belowcaptionskip}{3pt}}
\AddToHook{env/table*/begin}{\setlength{\belowcaptionskip}{3pt}}
\newcommand{\atd}{\textnormal{\textsc{ATD}}\xspace}

\newtheorem{proposition}{Proposition}
\newcommand{\preprintcodeavailability}{}
\input{generated/generated_tangent_macros}

\title{One Patch, Three Roles:\\
What Is Actually Coupled in\\
Autoregressive Time-Series Forecasting?}

\ifdefined\arxivversion%
  \input{arxiv/frontmatter}
\else
  \author{Anonymous Authors}
\fi

\begin{document}
\maketitle
\ifdefined\arxivversion%
  \pagestyle{plain}
  \thispagestyle{plain}
\fi

\begin{abstract}
Patch-based autoregressive time-series forecasting often ties input
representation, learned transitions, and recursive execution to one patch
length.  We ask which of these roles can be adjusted separately.  A supporting
atomic-encoding study finds greater sensitivity to model width than to atom
grouping on the evaluated grid.  Our main finding is that a frozen parent's recursive
trajectory is easier to fit than the observed future with lightweight parallel
exits.  Autoregressive Trajectory Distillation (\atd{}) turns this into
selectable ATD-1/2/4/8 execution, with ATD-1 exactly recovering the parent.  On
a paired four-data-set comparison, ATD-8 reaches
\generatedKEightSpeedup$\times$ end-to-end speedup with stable quality across
widths.  Fewer calls do not automatically remove the parent's existing
forecast error: \atd{} improves trajectory fidelity in all 21 seed runs but
forecast accuracy in only 15 against matched clean-future supervision.
We further find a correctable residual projection along a train-selected
periodic history direction.  Spectrum Tangent applies this correction without
adding neural parameters or Transformer calls.  At horizon 720, it
reduces mean squared error (MSE) and mean absolute error (MAE) by
\tangentHSevenTwentyReduction\% and
\tangentHSevenTwentyMaeReduction\% over seven data sets and two output widths,
while remaining \tangentSpeedupQOne$\times$ faster than recursive inference.
Level and shape projections sometimes disagree.  Trajectory compressibility,
the fidelity--accuracy mismatch, and the correction recur across three public AR parents.  Together these
results separate representation, transition, and execution as AR design axes.
\preprintcodeavailability{}
\end{abstract}

\input{sections/introduction_generated_tangent}
\input{sections/related_work_generated_tangent}
\input{sections/method_generated_tangent}
\input{sections/experiments_generated_tangent}
\input{sections/limitations_generated_tangent}
\input{sections/conclusion_generated_tangent}

\ifdefined\arxivversion\else\clearpage\fi
\section*{Reproducibility statement}
All local configurations use seeds 2021--2023.  Parent and decoder checkpoints
are validation-selected.  Tangent selects its period and coefficient from
training origins alone: blocks 1--3 fit and block 4 confirms the frozen choice;
test labels never affect selection.  We report all test origins in four
reporting-only blocks.  The appendices document checkpoint lineage,
normalization, parameters, calls, timing hardware, selectors, hashes, and
negative cases.

\section*{Statement on the use of AI tools}
Generative AI tools assisted code inspection, experiment orchestration,
plotting, and language editing.  The authors designed the method and
experiments and verified the implementation, numerical evidence, citations,
and claims.

\bibliographystyle{iclr2027_conference}
\bibliography{references}

\appendix
\input{sections/appendix_generated_tangent}

\end{document}

%% file: generated/generated_tangent_macros.tex
\newcommand{\planASixteenMatchedCells}{15}
\newcommand{\planASixteenMatchedWins}{9}
\newcommand{\planASixteenMatchedMacroRelative}{-0.51}
\newcommand{\pTwelveMRangeMse}{2.85}
\newcommand{\pTwelveMRangeMae}{2.03}
\newcommand{\pTwelveDRangeMse}{5.08}
\newcommand{\pTwelveDRangeMae}{4.16}

\newcommand{\targetFitEnergyRatio}{0.150}
\newcommand{\targetFitGeneratedUnexplained}{0.109}
\newcommand{\targetFitTruthUnexplained}{0.845}
\newcommand{\atdPostTrainMinutesMin}{0.17}
\newcommand{\atdPostTrainMinutesMax}{5.38}
\newcommand{\targetFitCellWins}{42}
\newcommand{\targetFitBlockWins}{168}

\newcommand{\timesfmTargetFitCellWins}{12}
\newcommand{\timesfmTargetFitBlockWins}{48}

\newcommand{\fourHorizonJointKFourMse}{0.3121}

\newcommand{\fourHorizonJointKEightMse}{0.3225}

\newcommand{\generatedHSevenTwentyMean}{0.371391}
\newcommand{\tangentHSevenTwentyMean}{0.361963}

\newcommand{\tangentHSevenTwentyRelative}{-2.54}
\newcommand{\tangentHSevenTwentyReduction}{2.54}
\newcommand{\generatedHSevenTwentyMaeMean}{0.377865}
\newcommand{\tangentHSevenTwentyMaeMean}{0.369045}

\newcommand{\tangentHSevenTwentyMaeRelative}{-2.33}
\newcommand{\tangentHSevenTwentyMaeReduction}{2.33}

\newcommand{\tangentHSevenTwentyBlocks}{143}
\newcommand{\tangentHSevenTwentyBlockTrials}{168}
\newcommand{\fourHorizonJointWidthDelta}{+0.010378}
\newcommand{\jointWidthSeedWins}{3}
\newcommand{\jointWidthSeedTrials}{21}
\newcommand{\jointWidthDatasetWins}{1}
\newcommand{\generatedVsJointKEightSeedWins}{19}
\newcommand{\generatedVsJointKEightDatasetWins}{6}
\newcommand{\jointKFourAverageSeedStd}{0.0069}
\newcommand{\jointKEightAverageSeedStd}{0.0124}
\newcommand{\generatedKFourAverageSeedStd}{0.0016}
\newcommand{\generatedKEightAverageSeedStd}{0.0017}
\newcommand{\tangentWrapperOverhead}{15.08}

\newcommand{\tangentSpeedupQOne}{3.24}
\newcommand{\generatedKFourSpeedup}{3.22}
\newcommand{\generatedKEightSpeedup}{5.54}
\newcommand{\namedParentKOneExact}{36}
\newcommand{\namedParentGeneratedKEightWins}{29}
\newcommand{\namedParentTangentKEightWins}{36}
\newcommand{\namedParentDirectTruthWins}{27}
\newcommand{\namedParentDirectFidelityWins}{36}

\newcommand{\namedParentTangentAllWidthWins}{108}
\newcommand{\namedParentTangentAllWidthTrials}{108}
\newcommand{\namedParentTangentAllWidthBlocks}{377}
\newcommand{\namedParentTangentAllWidthBlockTrials}{432}

\newcommand{\tangentNoPoolingWins}{42}
\newcommand{\tangentNoPoolingTrials}{42}
\newcommand{\tangentNoPoolingPeriodMatches}{42}
\newcommand{\tangentNoPoolingCoefficientRatioMin}{0.887}
\newcommand{\tangentNoPoolingCoefficientRatioMax}{1.124}
\newcommand{\tangentLeaveSeedOutWins}{42}
\newcommand{\tangentLeaveSeedOutTrials}{42}
\newcommand{\tangentLeaveSeedOutPeriodMatches}{42}
\newcommand{\tangentLeaveSeedOutCoefficientRatioMin}{0.941}
\newcommand{\tangentLeaveSeedOutCoefficientRatioMax}{1.051}
\newcommand{\exchangeTangentKFourAlpha}{1.496}
\newcommand{\exchangeTangentKEightAlpha}{1.518}
\newcommand{\exchangeTangentTrainConfirmations}{6}
\newcommand{\exchangeTangentKFourHSevenTwentyMseRelative}{+2.82}
\newcommand{\exchangeTangentKFourHSevenTwentyMaeRelative}{+1.22}
\newcommand{\exchangeTangentKEightHSevenTwentyMseRelative}{+4.26}
\newcommand{\exchangeTangentKEightHSevenTwentyMaeRelative}{+1.81}

\newcommand{\exchangeTangentHSevenTwentyBlockWins}{6}
\newcommand{\patchGeometryCells}{42}
\newcommand{\patchGeometryLevelPositive}{31}
\newcommand{\patchGeometryShapePositive}{33}
\newcommand{\patchGeometryFullPositive}{42}

\newcommand{\driftGeometryFarRows}{996}
\newcommand{\driftGeometryShapePositiveRows}{996}

\newcommand{\driftGeometryLevelIncreases}{33}
\newcommand{\driftGeometryCosineIncreases}{28}
\newcommand{\driftGeometryMarginIncreases}{34}
\newcommand{\driftGeometryDepthTrials}{36}
\newcommand{\driftGeometryKEightHSevenTwentyAboveShuffle}{35}
\newcommand{\driftGeometryKEightHSevenTwentyTrials}{36}

%% file: arxiv/frontmatter.tex
\iclrfinalcopy%
\renewcommand{\preprintcodeavailability}{%
  Code is available at \mbox{\url{https://github.com/RowanFFF/ATD-Spectrum-Tangent}}.
}
\author{%
  \normalfont{}Ziang Li$^{1}$, Yue Huang$^{1}$, Guoxu Zhou$^{1}$, Na Han$^{2}$,\\
  \normalfont{}Jie Wen$^{3}$, Lunke Fei$^{1}$, Xiaozhao Fang$^{1}$\thanks{Corresponding author.}\\[0.4em]
  \normalfont\small $^{1}$Guangdong University of Technology\\
  \normalfont\small $^{2}$Guangdong Polytechnic Normal University\\
  \normalfont\small $^{3}$Harbin Institute of Technology\\[0.3em]
  \normalfont\small\texttt{\{sspa\_131101,guoxu.zhou\}@qq.com}\\
  \normalfont\small\texttt{17324004911@163.com}\\
  \normalfont\small\texttt{\{hannagdut,jiewen\_pr,flksxm,xzhfang168\}@126.com}
}
\hypersetup{%
  pdftitle={One Patch, Three Roles: What Is Actually Coupled in Autoregressive Time-Series Forecasting?},
  pdfauthor={Ziang Li, Yue Huang, Guoxu Zhou, Na Han, Jie Wen, Lunke Fei, Xiaozhao Fang},
  pdfsubject={Autoregressive time-series forecasting; ATD and Spectrum Tangent}
}

%% file: sections/introduction_generated_tangent.tex
\section{Introduction}
\label{sec:introduction}

Autoregressive (AR) forecasting repeatedly predicts a future block and feeds
it back as context for later predictions.  Patch-based Transformers package
consecutive observations into tokens, supporting next-patch supervision and
flexible forecast horizons \citep{liu2024timer,liu2025timerxl,
liu2024autotimes,das2024timesfm}.  A general forecasting model must accommodate
data with different temporal scales while allowing efficient long-horizon
execution.

A patch plays a broader role here than in a standard Vision Transformer for
image classification \citep{dosovitskiy2021vit}.  It sets input token granularity
and, in a common next-patch AR design, also the supervised target and recursive
writeback span.  For example, forecasting 96 points with 24-point outputs takes
four recursive calls; 48-point outputs take two.  Larger outputs reduce
feedback steps, but also change the learning problem, so fewer calls alone do
not guarantee lower error.  Patch length thus links representation, training,
execution cost, and the path along which forecast errors accumulate.

\begin{figure*}[t]
  \centering
  \includegraphics[width=0.78\textwidth]{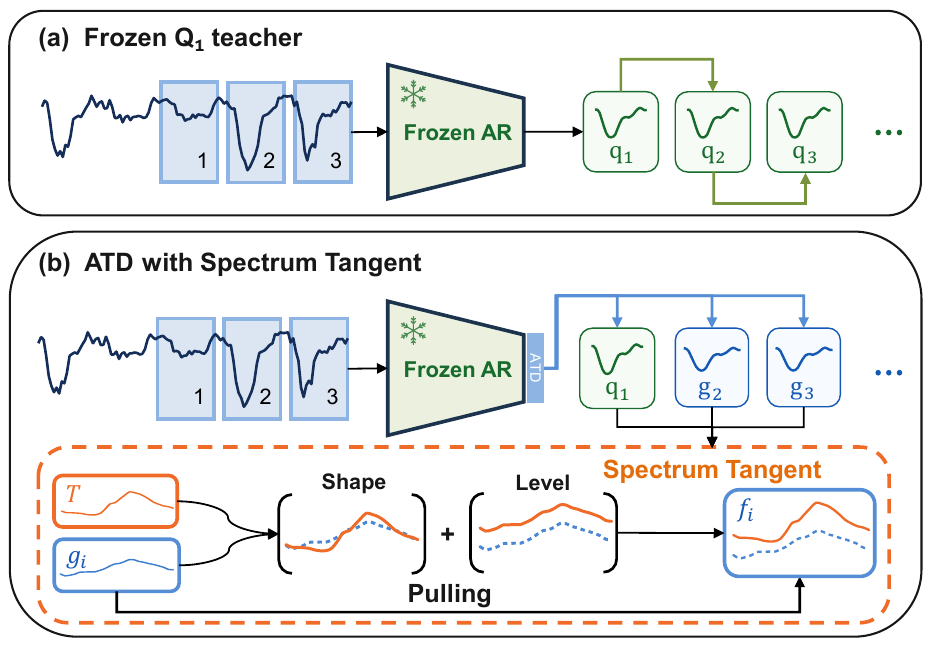}
  \caption{\textbf{The execution-decoupling branch and its residual correction.}
  \textbf{(a)} The frozen one-patch autoregressive (AR) parent (snowflake)
  recursively produces teacher patches $q_i$ from the patch-partitioned
  history.  \textbf{(b)} Autoregressive Trajectory Distillation (\atd{}) adds
  lightweight frozen-parent exits and emits
  $G=(q_1,g_2,\ldots,g_K)$ in one call, with $g_i$ trained to imitate $q_i$.
  Spectrum Tangent moves later proposals toward a train-selected history
  template $T$ before writeback.  The level and mean-free shape projections
  shown here are diagnostic views of that single displacement, not separately
  fitted correction modules.  Dashed blue and solid orange curves denote
  the proposal and its template-directed counterpart; no future label is used
  at inference.  The figure expands the execution branch from \atd{} to the
  fidelity--accuracy mismatch and its correction.  The supporting atomic-encoding
  observation is evaluated in \cref{sec:p12-results,app:p12-interface}.}
  \label{fig:overview}
  \vspace{-0.5em}
\end{figure*}

We ask: \emph{which roles assigned to an AR patch are actually coupled?}  Let
$p_0$ denote the representation atom, $P$ the learned transition span---including
its training target and recursive writeback---and $K$ the number of transitions
committed per call, for a $KP$-point output block.  The conventional monolithic,
one-patch implementation imposes two constraints:
\begin{equation}
  \underbrace{p_0=P}_{\text{representation = transition}},\qquad
  \underbrace{K=1}_{\text{transition = commit}}.
  \label{eq:core-separations}
\end{equation}
We relax these constraints one at a time, testing whether fixed atoms can
encode a transition patch and whether one call can emit several learned
transitions.

On the encoding side, adaptive patching and frequency-specific projections
illustrate the importance of temporal scale for heterogeneous data
\citep{chen2024pathformer,woo2024moirai}.  Our supporting study represents each
patch using shared 12-point atoms and chronological concatenation.  This
preserves the raw linear projection class; matched training gives small, mixed
changes.  More notably, a four-data-set grid shows greater sensitivity to
downstream model width $D$ than to atom grouping $m=P/12$, with different data
sets favoring different widths.  This motivates considering grouping and model
capacity separately when designing a shared forecasting interface.

The main empirical chain begins on the execution side.  We freeze a
validation-selected one-transition AR parent and find that its recursive
trajectory is a much easier target for lightweight parallel exits than the
observed future.  Autoregressive Trajectory Distillation (\atd{}) turns this
compressibility into a reusable compiler: exits for patches $2{:}K$ imitate the
parent's writeback trajectory while patch one remains the parent output.  One
compiled model supports ATD-1, ATD-2, ATD-4, and ATD-8 output blocks, with
ATD-1 exactly recovering the parent,
little forecast-quality variation across widths in the matched comparison, and
substantial acceleration.  Thus the parent transition can remain fixed and
available while deployment commits become wider.

Fewer calls, however, need not remove errors already present in the parent
trajectory.  In the 96-point example, ATD-2 can approximate the four-step parent
rollout in two calls, including its errors.  Against matched clean-future
supervision, parent-trajectory supervision improves rollout fidelity in all 21
seed runs but forecast accuracy in only 15.  This separates efficient execution
from correction of the inherited forecast residual.  Direct clean-future exits
are less stable on the matched grid.  Inspired by classical seasonal
forecasting, Spectrum Tangent tests whether a periodic template from observed
history supplies a correctable direction.  It fits one nonnegative coefficient
along this direction, outside the frozen parent, and adds no neural parameters
or Transformer calls.  Level and mean-free shape projections provide diagnostic
views: their signs can disagree even when the complete direction aligns with
the residual.  Correction gains strengthen at longer horizons.

\begin{samepage}
Our contributions are:
\begin{itemize}
  \item We empirically examine the separation of representation, learned
  transition, and per-call execution.  A supporting atomic-encoding study finds
  greater sensitivity to model width than to atom grouping on the evaluated
  grid, while trajectory compilation enables wider execution with a frozen
  learned transition.  Together these observations distinguish the roles
  commonly assigned to one patch.
  \item We identify the frozen parent's recursive trajectory as a readily
  compressible target and show that fitting it does not necessarily improve
  accuracy against truth.  \atd{} operationalizes the first finding as
  selectable ATD-1/2/4/8 execution; ATD-1 exactly recovers the parent, while
  wider modes provide measured acceleration with near-constant forecast quality
  in the matched comparison.
  \item We identify a correctable projection of the long-horizon residual along
  a periodic history direction.  Spectrum Tangent operationalizes this finding;
  level and shape projections provide complementary diagnostic views.  Tests
  across seven local data sets and three public AR parents document both the
  positive long-horizon result and its counterexamples.
\end{itemize}
\end{samepage}

%% file: sections/related_work_generated_tangent.tex
\section{Related work}
\label{sec:related}

\paragraph{Long-horizon forecasting.}
Long-horizon models use sparse attention, decomposition, patching, multiscale
mixing, or cross-variable structure
\citep{zhou2021informer,wu2021autoformer,zhou2022fedformer,nie2023patchtst,
liu2024itransformer,wang2024timemixer}; basis-expansion and MLP models offer
alternatives \citep{oreshkin2020nbeats,challu2023nhits,das2023tide}.  These
methods optimize clean-future forecasts; our question concerns the roles
inside a frozen AR patch interface.  Simple baselines can also expose
architectural assumptions that benchmark progress has left entangled, as
DLinear did for long-horizon Transformer
comparisons \citep{zeng2023dlinear}.  Our question is complementary to backbone
design: which roles assigned to an AR patch are actually inseparable?

\paragraph{Autoregressive and foundation forecasters.}
DeepAR uses probabilistic autoregressive forecasting \citep{salinas2020deepar};
Timer/Timer-XL, AutoTimes, and TimesFM extend next-patch generation to modern
and pretrained forecasters \citep{liu2024timer,liu2025timerxl,
liu2024autotimes,das2024timesfm}.  Universal and zero-shot models broaden this
direction \citep{ansari2024chronos,woo2024moirai,shi2025timemoe,
liu2025sundial,auer2025tirex}.  Moirai 2.0 predicts multiple tokens
\citep{liu2025moirai2}, whereas Timer-S1 scales serial prediction
\citep{liu2026timers1}; SE-LLM is a recent 672-point-context/96-point-output reference
\citep{liu2026sellm}.  These systems establish AR and multi-token forecasting
as useful interfaces.  Our question is how far representation and execution
can vary while retaining a selected parent's learned transition.

\paragraph{Parallel autoregressive decoding.}
Blockwise decoding, multi-token prediction, and auxiliary heads produce
multiple future tokens per call
\citep{stern2018blockwise,gloeckle2024multitoken,cai2024medusa}.
Speculative decoding uses a draft with online target verification
\citep{leviathan2023speculative}; STRIDE adapts this design to continuous
time-series patches \citep{subbaraman2025stride}.  \atd{} instead makes ATD-1
an exact bypass and removes the online parent pass for wider commits.  We report
rollout fidelity and forecast accuracy separately; the protocol and
operating-point comparison is in
\cref{tab:parallel-positioning,tab:stride-operating-points}.

Jacobi Forcing post-trains on its own parallel Jacobi decoding trajectories
\citep{hu2026jacobiforcing}; \atd{} distills numeric patches together with their
recursive writeback.

\paragraph{Exposure and rollout supervision.}
Classical multi-step forecasting distinguishes recursive, direct, DirRec,
MIMO, and DIRMO strategies \citep{bentaieb2012strategies}.  \atd{} revisits
this taxonomy at deployment: it retains the frozen recursive transition but
emits selectable blocks.  Unlike direct or MIMO forecasting, later outputs
target the parent's rollout rather than clean futures.  Boosting, scheduled
sampling, and dataset aggregation instead address error or exposure mismatch
\citep{bentaieb2014boosting,bengio2015scheduled,ross2011dagger}; rollout-aware
refinement also appears in neural simulation \citep{lippe2023pderefiner}.
Time-series distillation transfers teacher representations or predictions
to a separate forecasting model \citep{guo2026tllm,fu2026rednet}.  \atd{} instead
freezes a validation-selected AR parent, compiles its deployed patch path into
selectable widths, and preserves the original one-transition path as ATD-1.

\paragraph{Patch roles and spectral structure.}
PatchTST demonstrates the value of patch tokens \citep{nie2023patchtst}.
Pathformer adapts pathways across patch scales \citep{chen2024pathformer}, and
Moirai uses multiple patch-size projections to accommodate different
frequencies \citep{woo2024moirai}.  These approaches make temporal grouping an
explicit adaptation choice.  Our supporting atomic-encoding study complements
them by comparing sensitivity to grouping and model width.  TimesFM already
allows different input and output patch lengths \citep{das2024timesfm}; \atd{}
examines widening execution while fitting a frozen parent's recursive path.
Classical seasonal methods expose macroscopic level
and seasonal structure
\citep{winters1960forecasting,cleveland1990stl}, while frequency-aware models
learn it inside the backbone \citep{wu2021autoformer,zhou2022fedformer,
wu2023timesnet}.  Spectrum Tangent instead applies one external direction from
an immutable history template and a period selected on training trajectories.
Level and mean-free shape are diagnostic coordinates of this direction, rather
than separate learned forecasting components.

%% file: sections/method_generated_tangent.tex
\section{Decoupling patch roles}
\label{sec:method}

\subsection{Coupled baseline}

Let $h_t\in\mathbb{R}^{C\times W}$ be the current history with $C$ variables
and $W$ observed points, $P$ the number of points predicted by the original
one-transition forecaster, and $K$ the number of successive transitions emitted
per model call.  Let $U(h,p)$ append patch $p$ and retain the latest $W$ points.
The frozen parent $f_\theta$ includes instance normalization and its inverse,
and returns $\hat p_{t,1}\in\mathbb R^{C\times P}$ in the history's data coordinates:
\begin{equation}
  \hat p_{t,1}=f_\theta(h_t),\qquad
  h_{t+1}=U(h_t,\hat p_{t,1}).
  \label{eq:parent-transition}
\end{equation}
Local parents and exits are channel-independent with shared weights across
variables; no cross-variable mixing is used.  We train $f_\theta$ from scratch
using next-patch MSE and next-patch validation; \atd{} then freezes every
parent parameter.  ``Timer-style'' names this transition/writeback contract,
not the exact Timer architecture.

Execution may commit $K$ such transitions at once, spanning $KP$ points.  For
forecast horizon $H$, the number of structural backbone calls is
\begin{equation}
  C(H;P,K)=\left\lceil\frac{H}{KP}\right\rceil.
  \label{eq:calls}
\end{equation}

We separate representation from transition, then transition from execution,
measuring both parent-trajectory fidelity and forecast accuracy.

\subsection{An encoding-side observation}

Against a monolithic baseline $R:\mathbb R^P\!\to\!\mathbb R^D$, we let
$p_0=12$ and $m=P/p_0$.  One shared linear map
$\phi:\mathbb R^{12}\!\to\!\mathbb R^{16}$ lifts the $m$ chronological atoms;
we concatenate them without pooling and project the $16m$ coordinates to the
data-set-selected width $D=d_{\mathrm{model}}$.  The atom and lift widths are global
interface constants; each parent learns its own $\phi$ shared across atom
positions and chooses $P$ (hence $m$) and $D$.

If $\phi$ has full column rank, the concatenated features retain the ordered
raw patch before the final projection, and every raw linear patch map remains
representable.  \Cref{sec:p12-results} compares
matched training and the sensitivity along $m$ and $D$;
\cref{app:p12-interface} gives the proof and complete grid.

\subsection{Recursive trajectories as compilation targets}
\label{sec:atd}

From each training forecast origin $h_0$, we construct the frozen parent's
deployed teacher trajectory:
\begin{equation}
  h_0^Q=h_0,\quad q_1=f_\theta(h_0^Q),\quad
  h_j^Q=U(h_{j-1}^Q,q_j),\quad
  q_{j+1}=f_\theta(h_j^Q),
  \qquad j=1,\ldots,K-1.
  \label{eq:teacher-rollout}
\end{equation}
For the local parent, each call recomputes per-series mean and standard deviation over the
current $W$-point window, predicts in this normalization frame, maps the patch back for
writeback, shifts, and renormalizes.  This is additional to global benchmark
scaling.  Before one-shot regression, recursive teacher patches are
re-expressed in the initial origin's normalization frame as
$\bar q_j=(q_j-\mu_0)/\sigma_0$, channelwise.
The local state $s_0\in\mathbb R^{C\times D}$ is the last observed token after
the final Transformer block's post-normalization, before the output projection.
Named-parent readouts are specified in \cref{app:named-parent-states}.
An independent residual multilayer perceptron (MLP) $r_j$ predicts each later patch:
\begin{equation}
  \bar g_1=\bar q_1,\qquad
  \bar g_j=\bar q_1+r_j(s_0)\in\mathbb R^{C\times P},\quad j=2,\ldots,K.
  \label{eq:generated-exits}
\end{equation}
Only the $r_j$ are trained, using the origin-normalized rollout targets:
\begin{equation}
  \mathcal L_{\mathrm{ATD}}
    =\frac{1}{CP(K-1)}\sum_{j=2}^{K}
       \left\|\bar g_j-\bar q_j\right\|_F^2.
  \label{eq:atd-loss}
\end{equation}
The data-space proposal $G=[g_1\Vert\cdots\Vert g_K]$ uses
$g_j=\mu_0+\sigma_0\bar g_j$.  We call this compiled model \atd{} and denote
execution that emits $k$ transitions per call by ATD-$k$; ATD-1 bypasses every
auxiliary exit and exactly recovers the frozen parent.
Teacher states and recursive targets may be cached before fitting the
later-patch exits.

\paragraph{Matched target controls.}
Frozen Direct uses the same frozen parent and independent residual exits, but
replaces $\bar q_j$ in \cref{eq:atd-loss} by the identically normalized clean
future patch $\bar p_j^\star$.  Thus
target is the only changed factor.  Joint Direct also uses clean targets but
jointly trains the parent-initialized backbone and $K$ learned future placeholders
to decode $K$ clean-future patches in one causal pass.
These controls isolate the target and distinguish ordinary joint optimization.

\paragraph{Selectable widths.}
An ATD-8 compiled model contains exits $r_2,\ldots,r_8$.  Width $k$ evaluates only exits
$r_2,\ldots,r_k$ and writes back the resulting $k$ patches.  Local proposals
therefore share exact prefixes at a common history, although complete $k>1$
rollouts may diverge after different writebacks.

\medskip
\begin{proposition}[Exact ATD-1 fallback]
\label{prop:atd1}
The complete ATD-1 rollout is pointwise identical to the frozen parent rollout
at every horizon.
\end{proposition}
\begin{proof}
Both paths emit $f_\theta(h_0)$ and apply the same deterministic update $U$.
Their next histories are identical; induction proves equality thereafter.
\end{proof}

\subsection{A fidelity--accuracy mismatch}
\label{sec:fidelity-mismatch}

The parent trajectory already contains forecast error, including error
propagated through recursive writeback.  Fitting this trajectory can preserve
that error even when execution uses fewer calls.  Let $Q_\theta(h)$
denote the $K$-patch serial parent trajectory from
history $h$, $G(h)$ the ATD proposal, and $y$ the matching future.  Define the
rollout and forecast losses as
$\mathcal L_{\mathrm{roll}}(G)=\mathbb E\|G-Q_\theta\|_2^2$ and
$\mathcal L_{\mathrm{fore}}(G)=\mathbb E\|G-y\|_2^2$.  With
$e_{\mathrm{roll}}=G-Q_\theta$ and $e_Q=Q_\theta-y$, the identity
\begin{equation}
  \mathcal L_{\mathrm{fore}}(G)
  =\mathcal L_{\mathrm{roll}}(G)+\mathcal L_{\mathrm{fore}}(Q_\theta)
   +2\mathbb E\langle e_{\mathrm{roll}},e_Q\rangle
  \label{eq:error-decomposition}
\end{equation}
shows why closer imitation need not improve accuracy: parent error is inherited
and the cross term can have either sign.  \Cref{sec:experiments} tests this mismatch.

\subsection{Correcting truth error outside the frozen parent}
\label{sec:tangent}

Spectrum Tangent asks whether the residual $e=y-G$ has a correctable projection
along a simple history-derived direction while leaving the parent frozen.
Inspired by periodic forecasting, it forms a phase-wise template $T_\tau$ from
the normalized observed history at a train-selected, representation-aligned
period $\tau$.  Its phase means are fixed at the forecast origin; modulo-$\tau$
lookup aligns each future point, and the current call's mean and scale map them
to $G$'s coordinates (\cref{app:tangent-selection}).
``Spectrum'' refers to selecting $\tau$ over a period grid, while
``Tangent'' denotes the resulting one-dimensional correction of the frozen
proposal.  For absolute forecast endpoint $\ell$, the complete correction
direction is
\begin{equation}
  d_\tau(G,\ell)=b(\ell)\bigl(T_\tau-G\bigr),
  \label{eq:tangent-direction}
\end{equation}
where $b(\ell)$ is a dimensionless scalar, constant within each patch, taking
$0.25,0.5,1,2,3$ in successive bands separated by endpoints 96, 192, 336, and 672
(\cref{eq:endpoint-ramp}).  Here $\ell$ is measured from the forecast origin;
every call's first slot has zero direction.  Train-only
origins determine
\begin{equation}
  A_\tau=\mathbb E\|d_\tau\|_2^2,\quad
  B_\tau=\mathbb E\langle d_\tau,y-G\rangle,\quad
  V=\mathbb E\|y-G\|_2^2.
  \label{eq:spectrum-moments}
\end{equation}
Among periods attaining the largest positive one-direction explained error
$\max(B_\tau,0)^2/(A_\tau V)$, the selector chooses the smallest and freezes
\begin{equation}
  \alpha^*=\max(B_{\tau^*}/A_{\tau^*},0).
  \label{eq:alpha}
\end{equation}
This is the exact nonnegative minimizer of a convex quadratic
(derivation and selection protocol in \cref{app:tangent-selection}).
The frozen rule applies to slots 2--K:
\begin{equation}
  F=G+\alpha^* d_{\tau^*}(G,\ell).
  \label{eq:tangent-inference}
\end{equation}
The minimizer is exact on the fixed proposal trajectory used to estimate its
moments; closed-loop benefit remains empirical because corrected writeback
changes later histories.  The explicit wrapper uses only observed history and
ATD proposals.  It adds no trainable neural parameters or Transformer calls,
although phase lookup and arithmetic add measurable latency.  The inference
rule uses the complete displacement in \cref{eq:tangent-direction};
\cref{sec:tangent-results} later decomposes it into orthogonal level and
mean-free shape coordinates only to diagnose where its alignment comes from.

%% file: sections/experiments_generated_tangent.tex
\section{Experiments}
\label{sec:experiments}

The experiments follow the empirical chain.  We ask: (RQ1) how do atomic
grouping and model width affect forecasting after factorizing the input patch;
(RQ2) is a frozen transition's recursive trajectory sufficiently compressible
to support wider output blocks with near-constant forecast quality; (RQ3) does
improved rollout fidelity order forecast accuracy, and does the remaining
residual align with a history-derived direction; and (RQ4) do trajectory
compilation, the intervening mismatch, and the correction recur across named AR
parents and data sets?

\subsection{Protocol}

\paragraph{Data, selection, and reporting.}
We use ETTh1, ETTh2, ETTm1, ETTm2~\citep{zhou2021informer}, Weather, Electricity (ECL), and Traffic with lookback
$W=672$, training-set global standardization, and seeds 2021--2023.  We report
rolling horizons $H\in\{96,192,336,720\}$ over every chronological test origin.
We abbreviate these as H96, H192, H336, and H720.
Parents use next-patch validation and multi-transition decoders use closed-loop H720
validation.  Tangent uses one train-only period and coefficient per
data-set--width configuration, shared across its three parent seeds; chronological
selection and pooling controls are detailed in \cref{app:tangent-selection}.
All data-set-specific parent and tangent settings are deferred to
\cref{tab:dataset-config}.

\paragraph{Metrics and systems.}
Let $i$, $s$, $K$, and $b$ index data set, seed, commit width, and chronology
block.  We call $(i,K)$ a configuration, $(i,K,s)$ a seed run, and
$(i,K,s,b)$ a seed--block.  Within the named split, time-ordered origins are
partitioned into four contiguous, near-equal-count blocks.  Training blocks
serve the selector described above; test seed--blocks are reporting-only and
never affect selection.  Block counts are descriptive slices, not independent
random trials.  Mean squared error (MSE) and mean absolute error (MAE) use
standardized benchmark coordinates.
``Equal-configuration'' is an unweighted mean over the stated $(i,K)$
configurations after averaging all three seeds and test origins.
Batch-one RTX5880-Ada-48Q timing includes rolling normalization, writeback,
exits, and wrapper arithmetic; structural calls and milliseconds remain
separate.

\begin{table*}[t]
  \centering
  \caption{Parent qualification and local evidence chain under W672 rolling.
  MSE/MAE ($\downarrow$) use each source's reported H96/H192/H336/H720 Avg.
  row for external methods and the corresponding unrounded mean for local runs.
  $\dagger$ methods are transcribed from Timer-XL Table~12~\citep{liu2025timerxl};
  the local group shows three-seed means over all test origins, with seed standard
  deviations in \cref{tab:joint-direct-stability}.  External and local rows
  are grouped by provenance because their parent-patch protocols differ
  (\cref{tab:dataset-config}); the external block provides context, not a
  controlled ranking against the local rows.  Bold and underline mark the best
  and second-best distinct metric at displayed precision within each group;
  displayed ties share a mark.}
  \label{tab:four-horizon-benchmark}
  \scriptsize
  \setlength{\tabcolsep}{0.35pt}
  \input{generated/table_generated_tangent_benchmark_average.tex}
\end{table*}

\Cref{tab:four-horizon-benchmark} first establishes the frozen transition as a
reasonable local AR parent.  The external block is contextual rather than a
controlled comparison with the local rows.  Following the
compact convention of \citet{liu2025timerxl}, it provides four-horizon context
from Timer, UniTST, iTransformer, and PatchTST
\citep{liu2024timer,liu2024unitst,liu2024itransformer,nie2023patchtst};
the appendix restores DLinear~\citep{zeng2023dlinear} and the other reported
  baselines, together with every per-horizon MSE/MAE entry.  The matched evidence
  below follows the two separations and the correction of inherited forecast
  error.

\subsection{Atomic encoding: sensitivity to grouping and model width}
\label{sec:p12-results}

A shared full-rank lift of 12-point atoms to 16 coordinates, followed by
chronological concatenation, preserves every raw linear patch projection
(\cref{prop:a16-equivalence}); all 21 selected local checkpoints satisfy the
rank condition.  Across five matched data sets, this atomic encoding wins
\planASixteenMatchedWins/\planASixteenMatchedCells{} seed runs and changes pooled
H720 MSE by \planASixteenMatchedMacroRelative\%.  The mixed signs in
\cref{tab:p12-matched} show small, mixed changes after independent training.

The more distinctive observation concerns sensitivity.  In the four-data-set
$m\times D$ grid, the mean within-group normalized MSE/MAE range from
changing $m=P/12$ at fixed $D$ is
\pTwelveMRangeMse\%/\pTwelveMRangeMae\%, versus
\pTwelveDRangeMse\%/\pTwelveDRangeMae\% when changing $D$ at fixed $m$.
The larger variation lies along width, whose preferred value differs by data
set: ETTm2 and Weather attain their lowest MSE at $D=64$, ECL and Traffic at
$D=256$.  Width therefore warrants attention alongside patch grouping when
adapting representations to heterogeneous data.  Here $m$ also changes parent
span $P$ and $D$ changes backbone capacity; the full grid and interpretation
are in \cref{app:p12-interface}.

\subsection{Recursive trajectories can be compiled into wider outputs}

The parent trajectory is a substantially easier matched compilation target.
On 512 train-only origins per seed, its mean squared magnitude is
\targetFitEnergyRatio$\times$ that of the clean future; more importantly, after
normalizing each fit by its own target magnitude, the unexplained fraction is
\targetFitGeneratedUnexplained{} versus \targetFitTruthUnexplained{} for clean futures,
lower in all \targetFitCellWins/42 seed runs and \targetFitBlockWins/168
training seed--blocks.  An independent TimesFM validation analysis repeats the
ordering in \timesfmTargetFitCellWins/12 runs and
\timesfmTargetFitBlockWins/48 blocks.  Together they identify the deployed
trajectory as the simpler matched target for the lightweight exits (full
evidence in \cref{tab:target-fit}).

\begin{table*}[t]
  \centering
  \caption{Two-stage deployment accounting.  (a) The paired ETTh1, Weather,
  ECL, and Traffic prefix comparison (three seeds) tests whether one compiled model
  preserves H720 forecast MSE as commit width changes; ATD-1 bypasses later-patch
  exits.  (b) The seven-data-set $\times$ two-width row records the cost of the
  later macro correction, analyzed in \cref{sec:tangent-results}.
  Rollout MSE measures error against the serial parent; speedup is relative to
  the parent, and retained speedup is relative to ATD without correction.}
  \label{tab:multiwidth}
  \small
  \renewcommand{\arraystretch}{1.05}
  \input{generated/table_generated_tangent_multiwidth.tex}
\end{table*}

One compiled model spans ATD-1 through ATD-8 with only $0.00056$ H720 forecast-MSE range in
the paired comparison; all prefixes are exact at a common origin and complete
ATD-1 rollouts equal the frozen parent in 12 seed runs.  ATD-4 and ATD-8 reach
\generatedKFourSpeedup$\times$ and \generatedKEightSpeedup$\times$ end-to-end
speedup.  ATD-1's 0.96$\times$ also separates exact fallback from measured
runtime.  Thus self-trajectory prediction widens the per-call output with nearly
constant forecasting performance in this comparison, while the frozen parent
transition remains unchanged and exactly available through ATD-1.

\subsection{Trajectory fidelity does not order forecast accuracy}

\begin{table*}[t]
  \centering
  \caption{ATD-8 matched-target control at H720.  \atd{} and Frozen Direct
  share the parent, exit topology, parameter budget, and selector; only supervision changes.
  Deltas are relative changes from Frozen Direct to ATD, so negative is better.
  Rollout error is measured against the serial parent.
  Bold rows improve fidelity but worsen forecast accuracy.
  Raw MSE values are in \cref{tab:target-control-raw}.}
  \label{tab:target-control}
  \small
  \renewcommand{\arraystretch}{1.05}
  \input{generated/table_generated_tangent_target_control.tex}
\end{table*}

The matched targets separate trajectory imitation from forecast accuracy.
\Cref{tab:target-control}
changes only later-patch supervision: \atd{} improves parent-trajectory fidelity in
21/21 seed runs but forecast MSE in only 15/21.  ETTm2 and Traffic are direct
counterexamples---every seed becomes more faithful while forecast MSE worsens.
Closer imitation can retain the parent's existing forecast errors even with
wider outputs: rollout fidelity does not order forecast accuracy.
The data-set means,
block counts, raw values, and the ETTh2 outlier are retained in
\cref{tab:target-control-raw}.

Joint Direct is the shared-backbone clean-future control in
\cref{tab:four-horizon-benchmark}.  On the four-horizon average, widening it
from Joint Direct-4 to Joint Direct-8 improves only \jointWidthSeedWins/\jointWidthSeedTrials{} seeds
and \jointWidthDatasetWins/7 data-set means, while MSE changes from
\fourHorizonJointKFourMse{} to \fourHorizonJointKEightMse{}
(\fourHorizonJointWidthDelta).  Its mean seed standard deviation grows from
\jointKFourAverageSeedStd{} to \jointKEightAverageSeedStd{}, versus
\generatedKFourAverageSeedStd{}/\generatedKEightAverageSeedStd{} for
ATD-4/8.  ATD-8 wins \generatedVsJointKEightSeedWins/21 seeds
and \generatedVsJointKEightDatasetWins/7 means.  Thus clean-future widening is
less stable on this matched grid (full statistics in
\cref{tab:joint-direct-stability}).

\subsection{Long-horizon residual aligns with a macroscopic history direction}
\label{sec:tangent-results}

Spectrum Tangent addresses the forecast residual inherited from the parent
trajectory.  It tests whether a periodic, history-derived direction can reduce
this error outside the frozen parent.  At H720,
MSE falls from \generatedHSevenTwentyMean\ to \tangentHSevenTwentyMean\
(\tangentHSevenTwentyRelative\%) and MAE from \generatedHSevenTwentyMaeMean{} to
\tangentHSevenTwentyMaeMean{} (\tangentHSevenTwentyMaeRelative\%) across all 14
data-set--width means and 42/42 seeds.  Yet only
\tangentHSevenTwentyBlocks/\tangentHSevenTwentyBlockTrials{} chronology blocks
improve; the ETTm2 ATD-4 worst case is $+0.025758$.

\begin{table*}[t]
  \centering
  \caption{Horizon and cross-parent evidence.  (a) Tangent on the seven-data-set
  grid; W/L counts data-set--width improvements/regressions.  (b) Under each
  parent, subcolumns report ATD H720 forecast-MSE seed wins against the parent
  and +Tan. $\Delta$MSE (\%) against ATD.  Negative +Tan. values improve; all
  are 3/3 seed wins.  MSE is not pooled across parents.}
  \label{tab:tangent-evidence}
  \small
  \renewcommand{\arraystretch}{1.04}
  \begin{minipage}[t]{0.74\textwidth}
    \centering
    \textbf{(a) Horizon correction}\par\vspace{2pt}
    \small
    \setlength{\tabcolsep}{2.0pt}
    \input{generated/table_generated_tangent_horizons.tex}
  \end{minipage}
  \par\vspace{2pt}
  \begin{minipage}[t]{\textwidth}
    \centering
    \textbf{(b) Cross-parent transfer at H720}\par\vspace{2pt}
    \small
    \setlength{\tabcolsep}{1.8pt}
    \input{generated/table_generated_tangent_named_parent_crossdata.tex}
  \end{minipage}
\end{table*}

We use Exchange as a weak-seasonality stress test rather than an eighth
positive benchmark cell.  The unchanged train-only correction worsens H720
MSE/MAE by
\mbox{\exchangeTangentKFourHSevenTwentyMseRelative\%/\exchangeTangentKFourHSevenTwentyMaeRelative\%}
for Tangent-4 and
\mbox{\exchangeTangentKEightHSevenTwentyMseRelative\%/\exchangeTangentKEightHSevenTwentyMaeRelative\%}
for Tangent-8
(\cref{tab:tangent-exchange}).

Gains strengthen with horizon: H96 has six wins, two losses, and six exact ties
among 14 configurations, versus 14/14 improvements at H720.  This is consistent with the
accumulated-error interpretation (\cref{tab:tangent-evidence}; full counts in
\cref{tab:tangent-h720-internal}).

\paragraph{Diagnostic level--shape view.}
ETTh2 and ECL/Traffic exhibit opposing level--shape signs.  The complete
direction is positively aligned in 42/42 runs, while neither coordinate
alone describes all data sets; full results are in
\cref{tab:tangent-component-risk,tab:tangent-main-drift-geometry,tab:tangent-sensitivity}.

The wrapper adds no neural parameters or Transformer calls but incurs
\tangentWrapperOverhead\% latency overhead; it remains
\tangentSpeedupQOne$\times$ faster than recursive parent inference.

\subsection{The execution-side phenomena recur across named AR parents}

Across AutoTimes--GPT2, Timer-base-84m, and TimesFM-2.5-200M, all
\namedParentKOneExact/36 ATD-1 rollouts recover the parent.  Panel (b) of
\cref{tab:tangent-evidence} separates the two ATD-8 stages: ATD improves on the
recursive parent in \namedParentGeneratedKEightWins/36 seeds; Tangent then
improves the matched ATD in \namedParentTangentKEightWins/36
($-0.84\%$ to $-15.66\%$).  Across ATD-2/4/8, Tangent improves
\namedParentTangentAllWidthWins/\namedParentTangentAllWidthTrials{}
runs and \namedParentTangentAllWidthBlocks/\namedParentTangentAllWidthBlockTrials{}
blocks; the ATD-8 target control gives \namedParentDirectFidelityWins/36 fidelity but only
\namedParentDirectTruthWins/36 forecast wins (full results in
\cref{tab:named-parents}).

%% file: generated/table_generated_tangent_benchmark_average.tex
\begin{tabular*}{\textwidth}{@{\extracolsep{\fill}}l*{20}{r}@{}}
\toprule
Models & \multicolumn{10}{c}{Author-reported$^\dagger$} & \multicolumn{10}{c}{Local, three seeds} \\
\cmidrule(lr){2-11}\cmidrule(lr){12-21}
 & \multicolumn{2}{c}{Timer-XL} & \multicolumn{2}{c}{Timer} & \multicolumn{2}{c}{UniTST} & \multicolumn{2}{c}{iTrans.} & \multicolumn{2}{c}{PatchTST} & \multicolumn{2}{c}{Parent} & \multicolumn{2}{c}{Joint-4} & \multicolumn{2}{c}{Joint-8} & \multicolumn{2}{c}{ATD-8} & \multicolumn{2}{c}{+Tangent-8} \\
 & \multicolumn{2}{c}{(\citeyear{liu2025timerxl})} & \multicolumn{2}{c}{(\citeyear{liu2024timer})} & \multicolumn{2}{c}{(\citeyear{liu2024unitst})} & \multicolumn{2}{c}{(\citeyear{liu2024itransformer})} & \multicolumn{2}{c}{(\citeyear{nie2023patchtst})} & \multicolumn{2}{c}{(Recursive)} & \multicolumn{2}{c}{(Clean)} & \multicolumn{2}{c}{(Clean)} & \multicolumn{2}{c}{(Rollout)} & \multicolumn{2}{c}{(Corrected)} \\
Dataset & MSE & MAE & MSE & MAE & MSE & MAE & MSE & MAE & MSE & MAE & MSE & MAE & MSE & MAE & MSE & MAE & MSE & MAE & MSE & MAE \\
\midrule
ETTh1 & \textbf{0.409} & \textbf{0.430} & 0.418 & 0.436 & 0.429 & 0.447 & 0.421 & 0.445 & \underline{0.412} & \underline{0.435} & 0.362 & 0.405 & 0.409 & 0.427 & 0.430 & 0.439 & \underline{0.361} & \underline{0.403} & \textbf{0.355} & \textbf{0.398} \\
ETTh2 & \textbf{0.352} & \underline{0.402} & 0.382 & 0.418 & 0.384 & 0.428 & 0.389 & 0.421 & \underline{0.359} & \textbf{0.400} & \underline{0.307} & \underline{0.371} & 0.357 & 0.403 & 0.348 & 0.408 & 0.316 & 0.376 & \textbf{0.305} & \textbf{0.370} \\
ETTm1 & 0.359 & \textbf{0.382} & \underline{0.352} & \underline{0.383} & \underline{0.352} & 0.388 & 0.376 & 0.403 & \textbf{0.349} & 0.385 & 0.346 & 0.381 & 0.350 & 0.391 & 0.363 & 0.401 & \underline{0.342} & \underline{0.379} & \textbf{0.339} & \textbf{0.375} \\
ETTm2 & 0.271 & 0.322 & 0.275 & 0.327 & \underline{0.265} & \textbf{0.306} & 0.290 & 0.340 & \textbf{0.261} & \underline{0.318} & 0.282 & 0.330 & \underline{0.273} & \underline{0.328} & 0.281 & 0.332 & 0.277 & \underline{0.328} & \textbf{0.269} & \textbf{0.321} \\
ECL & \textbf{0.155} & \textbf{0.246} & \underline{0.161} & \underline{0.251} & 0.163 & 0.257 & 0.164 & 0.258 & 0.169 & 0.268 & \underline{0.153} & \underline{0.248} & 0.162 & 0.262 & 0.169 & 0.271 & \underline{0.153} & \underline{0.248} & \textbf{0.152} & \textbf{0.247} \\
Traffic & \textbf{0.374} & \textbf{0.255} & \underline{0.384} & \underline{0.259} & 0.385 & 0.265 & \underline{0.384} & 0.274 & 0.391 & 0.275 & \underline{0.384} & \underline{0.263} & \underline{0.384} & \underline{0.263} & 0.385 & 0.264 & 0.386 & 0.265 & \textbf{0.382} & \textbf{0.261} \\
Weather & 0.240 & 0.273 & 0.232 & \underline{0.270} & \underline{0.231} & 0.272 & 0.266 & 0.291 & \textbf{0.226} & \textbf{0.268} & \underline{0.235} & \underline{0.272} & 0.250 & 0.299 & 0.281 & 0.329 & \textbf{0.234} & \underline{0.272} & \textbf{0.234} & \textbf{0.271} \\
\midrule
Avg. & \textbf{0.309} & \textbf{0.330} & 0.315 & \underline{0.335} & 0.316 & 0.338 & 0.327 & 0.347 & \underline{0.310} & 0.336 & \underline{0.295} & \underline{0.324} & 0.312 & 0.339 & 0.323 & 0.349 & 0.296 & \underline{0.324} & \textbf{0.291} & \textbf{0.320} \\
\bottomrule
\end{tabular*}

%% file: generated/table_generated_tangent_multiwidth.tex
\begin{tabular*}{\textwidth}{@{\extracolsep{\fill}}lrrrrr@{}}
\toprule
\multicolumn{6}{@{}l@{}}{\textbf{(a) Selectable prefixes}} \\
Mode & Calls & Forecast MSE & Rollout MSE & Latency (ms) & Speedup \\
\midrule
ATD-1 & 15.25 & 0.3402 & 0.0000 & 23.19 & 0.96$\times$ \\
ATD-2 & 7.75 & 0.3405 & 0.0069 & 12.89 & 1.73$\times$ \\
ATD-4 & 4.00 & 0.3400 & 0.0089 & 6.93 & 3.22$\times$ \\
\textbf{ATD-8} & 2.00 & 0.3401 & 0.0112 & 4.03 & \textbf{5.54$\times$} \\
\midrule
\end{tabular*}
\par\nobreak\vspace{3pt}
\begingroup
\renewcommand{\arraystretch}{1.18}
\begin{tabularx}{\textwidth}{@{}*{4}{>{\centering\arraybackslash}X}@{}}
\multicolumn{4}{@{}l@{}}{\textbf{(b) Tangent wrapper timing}} \\
\addlinespace[2pt]
Wrapper & $\Delta$ latency & Speedup retained & Speedup vs parent \\
\midrule
\textbf{+Tangent} & +15.08\% & 86.89\% & 3.24$\times$ \\
\bottomrule
\end{tabularx}
\endgroup

%% file: generated/table_generated_tangent_target_control.tex
\begin{tabular*}{\textwidth}{@{\extracolsep{\fill}}lrrrr@{}}
\toprule
& \multicolumn{2}{c}{Forecast error} & \multicolumn{2}{c}{Rollout error} \\
\cmidrule(lr){2-3}\cmidrule(lr){4-5}
Dataset & $\Delta$MSE (\%) & Seed wins & $\Delta$MSE (\%) & Seed wins \\
\midrule
ETTh1 & -8.1 & 3/3 & -85.1 & 3/3 \\
ETTh2 & -81.0 & 3/3 & -99.3 & 3/3 \\
ETTm1 & -1.8 & 3/3 & -70.2 & 3/3 \\
\textbf{ETTm2} & \textbf{+3.5} & 0/3 & -45.8 & 3/3 \\
Weather & -2.3 & 3/3 & -69.4 & 3/3 \\
ECL & -2.3 & 3/3 & -79.1 & 3/3 \\
\textbf{Traffic} & \textbf{+0.9} & 0/3 & -69.2 & 3/3 \\
\bottomrule
\end{tabular*}

%% file: generated/table_generated_tangent_horizons.tex
\begin{tabular*}{\linewidth}{@{\extracolsep{\fill}}rrrr@{}}
\toprule
Horizon & Config. W/L & Seed wins & Mean $\Delta$MSE \\
\midrule
96 & 6/2 & 18/42 & -0.0011 \\
192 & 12/2 & 36/42 & -0.0034 \\
336 & 13/1 & 40/42 & -0.0061 \\
\textbf{720} & 14/0 & 42/42 & \textbf{-0.0094} \\
\bottomrule
\end{tabular*}

%% file: generated/table_generated_tangent_named_parent_crossdata.tex
\begin{tabular*}{\linewidth}{@{\extracolsep{\fill}}lrrrrrr@{}}
\toprule
& \multicolumn{2}{c}{AutoTimes} & \multicolumn{2}{c}{Timer} & \multicolumn{2}{c}{TimesFM} \\
\cmidrule(lr){2-3}\cmidrule(lr){4-5}\cmidrule(lr){6-7}
Dataset & ATD & +Tangent & ATD & +Tangent & ATD & +Tangent \\
& Seed wins & $\Delta$MSE (\%) & Seed wins & $\Delta$MSE (\%) & Seed wins & $\Delta$MSE (\%) \\
\midrule
ETTh1 & 0/3 & -1.56 & 3/3 & -1.93 & 3/3 & -2.04 \\
ETTh2 & 2/3 & -6.14 & 0/3 & -4.19 & 3/3 & -6.84 \\
ETTm1 & 3/3 & -0.84 & 3/3 & -15.66 & 3/3 & -5.06 \\
Weather & 3/3 & -3.07 & 3/3 & -5.83 & 3/3 & -12.32 \\
\bottomrule
\end{tabular*}

%% file: sections/limitations_generated_tangent.tex
\section{Scope and limitations}
\label{sec:limitations}

The scope is patch-based AR point forecasting.  Atomic encoding is evaluated
with separately trained parents; it does not establish a shared universal
encoder.  \atd{} requires teacher trajectories; only ATD-1 is
exact.  Tangent assumes recurrence: H96 retains regressions, 143/168 H720 blocks
improve, and Exchange degrades.  Horizon trends do not isolate recursive error
from other forecast errors.  Timing is hardware-specific and measures batch-one
latency; greater GPU utilization at larger batches may reduce the speedup.
Offline caching and exit fitting must be amortized over repeated forecasts
(\cref{app:additional-results}); named-parent evidence establishes compatibility,
not throughput.

%% file: sections/conclusion_generated_tangent.tex
\section{Conclusion}
\label{sec:conclusion}

An AR patch often couples representation, training transition, and execution.
Atomic encoding shows greater sensitivity to model width than to grouping on
the evaluated grid, while parent trajectories support wider \atd{} execution.
Fewer calls can retain the parent's existing errors: closer imitation does not
order truth accuracy.  Spectrum Tangent corrects an aligned component of the
long-horizon residual using periodic history; level and shape are diagnostic
views.  These findings clarify which patch roles can be separated, while showing
that more efficient execution and forecast-error correction remain distinct
objectives.

%% file: sections/appendix_generated_tangent.tex
\section{Experimental details and complete result tables}
\label{app:additional-results}
\small

\subsection{Complete four-horizon forecasting results}

The reported entries below are transcribed entry-for-entry from Timer-XL
Table~12~\citep{liu2025timerxl}.  It covers Timer, UniTST, iTransformer,
DLinear, PatchTST, TimesNet, Non-stationary Transformer, and Autoformer
\citep{liu2024timer,liu2024unitst,liu2024itransformer,zeng2023dlinear,
nie2023patchtst,wu2023timesnet,liu2022nonstationary,wu2021autoformer}.
These are author-reported comparisons, not local reruns.  We therefore keep their
P96-output protocol separate from our data-set-specific-$P$ local group even
though both use W672 rolling forecasts to H96/H192/H336/H720.

\begin{table*}[t]
  \centering
  \caption{Complete reported MSE/MAE entries transcribed from Timer-XL
  Table~12.  One source model per data set is trained at W672/P96 and rolled
  to four horizons; Avg. reproduces the source table's printed Avg. row rather
  than re-averaging its rounded entries.  Bold/underline mark the best/second-best
  distinct reported metric at displayed precision in each row; ties share a mark.}
  \label{tab:timerxl-full-results}
  \scriptsize
  \setlength{\tabcolsep}{0.6pt}
  \input{generated/table_generated_tangent_external_full.tex}
\end{table*}

\begin{table*}[t]
  \centering
  \caption{Complete local MSE/MAE entries.  Every entry averages all test
  origins and seeds 2021--2023; Avg. is the four-horizon mean, and
  bold/underline mark the best/second-best distinct local metric at displayed
  precision in each row; ties share a mark.
  Joint Direct uses the same parent initialization and closed-loop H720 selector;
  ATD and Tangent use the frozen selections reported in
  \cref{tab:dataset-config}.}
  \label{tab:local-full-results}
  \small
  \setlength{\tabcolsep}{0.9pt}
  \input{generated/table_generated_tangent_local_full.tex}
\end{table*}

\begin{table*}[t]
  \centering
  \caption{Three-seed four-horizon stability of the clean-future joint predictor
  and ATD.  For each seed, MSE/MAE is
  first averaged over H96/H192/H336/H720; entries then report mean$\pm$sample
  standard deviation over seeds 2021--2023.  The Avg. column is the
  equal-data-set mean and, after $\pm$, the mean of the seven data-set-wise
  standard deviations; it is not a standard deviation across heterogeneous
  data sets.}
  \label{tab:joint-direct-stability}
  \small
  \setlength{\tabcolsep}{1.5pt}
  \input{generated/table_generated_tangent_joint_stability.tex}
\end{table*}

Joint Direct appends learned future placeholders to the observed token
sequence and jointly updates them with the parent-initialized backbone against clean
future patches.  It is Moirai-2.0-style in its multi-token objective
\citep{liu2025moirai2}, but is a point-MSE control rather than a reproduction
of Moirai 2.0's probabilistic architecture.  On the four-horizon average, the
wider Joint Direct-8 control has larger average seed dispersion than Joint Direct-4 and improves only
\jointWidthSeedWins/\jointWidthSeedTrials{} seed pairs, whereas ATD's
frozen backbone keeps dispersion substantially smaller.
\FloatBarrier{}

\subsection{Data-varying configurations}

All local runs use $W=672$, training-set global standardization, seeds 2021--2023,
eight attention heads, and feed-forward width $4D$.  Parent checkpoints use
next-patch validation; ATD-4/8 checkpoints use closed-loop H720
validation.  \Cref{tab:dataset-config} records every data-varying parent
configuration and the frozen tangent selection shared by all three seeds in
each data-set--width configuration.  Each 12-point atom is lifted to 16 coordinates; $\alpha^*$ is the
effective coefficient applied in \cref{eq:tangent-inference}.  The candidate
period grid is the common set of 12-sample multiples through 216, the largest
multiple represented at least three times in $W=672$.  The global
$b(\ell)=(0.25,0.5,1,2,3)$ ramp at boundaries $(96,192,336,672)$ is fixed before
moment fitting and is unchanged across data sets, seeds, and $K$.  The grid and
ramp are design constants, not data-varying selected hyperparameters.

Local ATD uses channel-independent, zero-initialized two-layer MLP exits while
the parent remains in evaluation mode.  Local instance normalization uses
$\mu_c=W^{-1}\sum_t h_{c,t}$ and
$\sigma_c=\sqrt{W^{-1}\sum_t(h_{c,t}-\mu_c)^2+0.001}$; a bar denotes this
frame, distinct from training-set global standardization.
Across all 42 ATD/Direct seed runs,
complete ATD-1 differences are zero and full-origin H720 MSE matches the frozen
evaluation records within $1.46\times10^{-10}$.

The 42 local ATD-4/8 post-training jobs took
\atdPostTrainMinutesMin--\atdPostTrainMinutesMax{} minutes each in the campaign
logs, including process startup, teacher-state/target caching, exit fitting,
and validation, but excluding parent training and final testing.  These are
logged job wall times, not isolated GPU-time measurements.  Deployment savings
amortize this one-time cost only after roughly
$N>C_{\rm post}/(t_{\rm parent}-t_{\rm ATD})$ forecasts, when the denominator is
positive and all times refer to a matched hardware/batch configuration;
we do not infer a break-even count by mixing unmatched timing records.

\begin{table*}[!t]
  \centering
  \caption{Data-set-specific parent and tangent hyperparameters.  $P$, $D$,
  and $L$ are parent patch length, model width, and Transformer depth.  For
  each $K$, $\tau^*$ and $\alpha^*_K$ are fit after pooling seeds 2021--2023 on
  training blocks 1--3 and are shared by their three runs.  The separately
  selected local ATD-4/8 periods coincide, so panel (a) shows $\tau^*$ once.
  Panel (b) supplies all named-parent ATD-2/4/8 selections used in
  \cref{tab:named-parents,tab:tangent-drift-geometry}, including $\alpha^*_2$.
  No local ATD-2+Tangent result is reported.}
  \label{tab:dataset-config}
  \scriptsize
  \setlength{\tabcolsep}{2.6pt}
  \textbf{(a) Local parents}\par\vspace{2pt}
  \input{generated/table_generated_tangent_dataset_config.tex}
  \par\vspace{6pt}
  \textbf{(b) Named-parent tangent selections}\par\vspace{2pt}
  \input{generated/table_generated_tangent_named_config.tex}
\end{table*}
\FloatBarrier{}

\subsection{Named-parent state extraction}
\label{app:named-parent-states}

For each named parent, $s_0\in\mathbb R^{C\times D_{\rm parent}}$ contains
one final-token hidden vector per channel at the current ATD call origin.
The implementation folds channels into the batch axis; exit weights are
shared across channels, without cross-variable mixing.

\paragraph{AutoTimes-GPT2 ($D_{\rm parent}=768$).}
We select \texttt{last\_hidden\_state[:, -1]} from GPT2: the last observed
96-point token after the final Transformer block and GPT2's final
\texttt{ln\_f}, before the AutoTimes MLP decoder.

\paragraph{Timer-base-84m ($D_{\rm parent}=1024$).}
We select \texttt{hidden\_states[-1][:, -1]} from Timer's prediction output:
the last consumed 96-point token after the final decoder block and the
stack's final LayerNorm, before \texttt{lm\_head}.

\paragraph{TimesFM-2.5-200M ($D_{\rm parent}=1280$).}
We select \texttt{output\_embeddings[:, -1]} from the native forward pass:
the last consumed 32-point input token after the final Transformer block,
before the point and quantile output projections.  There is no additional
stack-final normalization.  This hidden vector feeds exits predicting
128-point output blocks; it is not the forecast or a pooled output block.

All three readouts enter the residual exit MLPs directly, without additional
state normalization, projection, or temporal/layer pooling.  The frozen
parent remains in evaluation mode.  AutoTimes recomputes the rolling-window
encoding at each call; Timer and TimesFM retain their native KV caches and
refresh the readout after consuming the committed patches.  Thus $s_0$ is
call-local, not a vector held fixed throughout the forecast.  KV caches and
normalization statistics remain part of parent execution, not extra exit inputs.

\FloatBarrier{}
\subsection{Additional analyses}

\begin{table*}[!htbp]
  \centering
  \caption{Matched-target fitting audits.  The local parent uses train-only origins
  over seven data sets and ATD-4/8; TimesFM uses a shared validation cache over
  four data sets at ATD-8.  Each has three seeds.  Energy ratio is parent trajectory
  over the clean future; unexplained fraction is fitted MSE divided by each
  target's own energy.  Runs/seed--blocks count a lower fraction for the parent trajectory.}
  \label{tab:target-fit}
  \scriptsize
  \renewcommand{\arraystretch}{1.05}
  \input{generated/table_generated_tangent_target_fit.tex}
\end{table*}

\begin{table}[h]
  \centering
  \caption{Raw values underlying \cref{tab:target-control}.  Entries average
  three seeds; ATD and Frozen Direct differ only in the later-patch target.}
  \label{tab:target-control-raw}
  \scriptsize
  \setlength{\tabcolsep}{1.3pt}
  \input{generated/table_generated_tangent_target_control_raw.tex}
\end{table}

The absence of a STRIDE row in the matched result tables is a protocol
boundary rather than an omitted nearest baseline.  Its Timer-XL operating
point couples a P96/context-1536 target, a 0.125$\times$ draft, an acceptance
temperature, and acceptance-dependent target verification.  Our local system
instead times one W672 compiled model with data-specific $P$ and fixed commit
widths.  A controlled systems comparison would need to hold fixed the accepted
parent checkpoint, draft construction, context and patch sizes, batch size,
hardware, accuracy gate, and all draft, verification, fallback, and wrapper
work.

\begin{table}[h]
  \centering
  \caption{Nearby parallel-generation designs.  ``Runtime target'' means
  invocation of the accepted autoregressive model after proposal.  STRIDE
  keeps the target frozen but runs a separate draft--verify path online.}
  \label{tab:parallel-positioning}
  \scriptsize
  \begin{tabularx}{\columnwidth}{lXcc}
    \toprule
    Method & Parallel signal & Post-train & Runtime target \\
    \midrule
    Moirai 2.0 & clean future patches & joint & no \\
    STRIDE & separate-draft proposals & no & verify \\
    Jacobi Forcing & Jacobi generation trajectory & yes & refine \\
    \textbf{\atd{}} & serial deployed parent trajectory & exits & no \\
    \bottomrule
  \end{tabularx}
\end{table}

\begin{table*}[t]
  \centering
  \caption{Closest acceleration reference without cross-protocol ranking.
  STRIDE$^\dagger$ entries are author-reported Timer-XL H720 points from arXiv v2,
  Table~1~\citep{subbaraman2025stride}: batch one, K3, acceptance temperature 0.25, and GPU-compute timing
  on one H200.  Local speed uses batch-one RTX5880-Ada-48Q timing on the
  four-data-set paired timing subset; local accuracy averages seven data sets
  and three seeds.  Thus rows expose the operating contracts but are not a
  controlled speed leaderboard.}
  \label{tab:stride-operating-points}
  \scriptsize
  \setlength{\tabcolsep}{3.5pt}
  \input{generated/table_generated_tangent_stride.tex}
\end{table*}

\begin{table}[h]
  \centering
  \caption{Transfer across three named AR parents and four data sets.
  (a) ATD-8 H720 counts are seeds out of three; Tangent delta is relative to the
  matched ATD model, and Frozen Direct changes only the later-patch
  target.  ATD-1 exact denotes bit-exact complete-rollout recovery (36/36 overall).
  (b) Test-block wins across ATD-2/4/8; each parent--width entry contains 48 blocks,
  and each width total contains 144 blocks.
  MSE is never pooled across parents.}
  \label{tab:named-parents}
  \scriptsize
  \textbf{(a) ATD-8 H720 seed counts}\par\vspace{2pt}
  \setlength{\tabcolsep}{2.8pt}
  \input{generated/table_generated_tangent_named_parents.tex}
  \par\vspace{4pt}
  \textbf{(b) Multiwidth test-block wins}\par\vspace{2pt}
  \setlength{\tabcolsep}{2.8pt}
  \input{generated/table_generated_tangent_transfer_stability.tex}
\end{table}

The four data sets are the common frozen compatibility grid across all three
named-parent campaigns: ETTh1, ETTh2, ETTm1, and Weather.  This study changes
the parent architecture rather than repeating the seven-data-set local
benchmark.  ETTm2, ECL, and Traffic have no named-parent results in this grid;
the transfer claim is limited to the four evaluated data sets.

With ATD-8, compilation improves over the recursive parent forecast MSE in
\namedParentGeneratedKEightWins/36 parent--data--seed runs; the matched target
control favors ATD in \namedParentDirectFidelityWins/36 fidelity but only
\namedParentDirectTruthWins/36 forecast comparisons.  Tangent then improves
ATD in \namedParentTangentKEightWins/36 runs; complete values are in
\cref{tab:named-parents}.

\section{Atomic encoding: construction and sensitivity}
\label{app:p12-interface}
\small

We fix $p_0=12$ because it is a common divisor of every studied parent span
$P\in\{24,48,96\}$; on hourly data it also corresponds to half a day, but it
has no universal physical-period meaning across the mixed sampling rates.  The
atom width 16 is likewise a fixed overcomplete lift that permits full column
rank, not a data-set-specific semantic choice or a separately tuned module.
In the tables, P12--16 denotes this 12-point atom and 16-coordinate lift.

For $m=P/12$ chronological atoms $u_q\in\mathbb R^{12}$, each parent learns one
map $\phi$ shared across atom positions and uses
\begin{equation}
  z_P=\psi_P[\phi(u_1)\Vert\cdots\Vert\phi(u_m)],\qquad
  \phi:\mathbb R^{12}\!\to\mathbb R^{16},\quad
  \psi_P:\mathbb R^{16m}\!\to\mathbb R^{d_{model}}.
  \label{eq:p12-interface}
\end{equation}

The following proposition records the class-preservation property used by the
encoding-side control.

\begin{proposition}[Information-preserving atomic factorization]
\label{prop:a16-equivalence}
If $\phi$ has full column rank, the concatenated atomic encoding is injective.
Moreover, for every raw linear patch projection
$R:\mathbb R^{12m}\to\mathbb R^{d_{model}}$ there exists a $\psi_P$ such that
\cref{eq:p12-interface} equals $R$ for every input.
\end{proposition}
\begin{proof}
Full column rank guarantees a left inverse $L\in\mathbb R^{12\times16}$
satisfying $L\phi=I_{12}$.  Let
$L_m=\operatorname{diag}(L,\ldots,L)$ and choose $\psi_P=RL_m$.
Then $\psi_P\operatorname{diag}(\phi,\ldots,\phi)=R$.
\end{proof}
Thus each ordered raw patch can be recovered from its atomic encoding, and
chronological concatenation preserves the raw linear function class.  This
establishes class preservation; optimization equivalence is evaluated
empirically rather than guaranteed by the construction.

This injectivity statement extends algebraically to multivariate patches:
applying the lift independently to $C$ variables gives a block-diagonal map
with left inverse $I_C\otimes L_m$.  Thus any linear cross-variable patch
projection can also be represented with a suitable subsequent joint map.
Our implemented parents instead share a per-variable projection and have no
cross-variable mixing.  The sensitivity results therefore do not establish
the same empirical behavior for cross-variable models, nor does a later
width-reducing projection necessarily preserve all input information.

\begin{table}[h]
  \centering
  \caption{Atomic input interfaces and a direct checkpoint audit of the
  sufficient condition in \cref{prop:a16-equivalence}.  The final column counts
  selected seed checkpoints whose learned $16\times12$ atom lift has full
  column rank.  This is a structural audit, not an accuracy comparison or an
  exhaustive search over atom sizes.}
  \label{tab:p12-interface}
  \scriptsize
  \setlength{\tabcolsep}{2.2pt}
  \input{generated/table_generated_tangent_plan_a16.tex}
\end{table}

\begin{table}[h]
  \centering
  \caption{Matched H720 representation control.  Raw and P12--16 entries are
  three-seed mean $\pm$ sample standard deviation of MSE; the final column
  averages seed-paired relative changes (negative favors P12--16).
  ETTm2$^\ast$ uses the stricter same-draw, exact-transplant initialization;
  both arms are then trained independently.  Other rows retain matched
  architectures with independent draws.}
  \label{tab:p12-matched}
  \scriptsize
  \setlength{\tabcolsep}{2.5pt}
  \input{generated/table_generated_tangent_plan_a16_matched.tex}
\end{table}

The matched ablation in \cref{tab:p12-matched} changes only the raw
patch projection versus the chronological atomic parameterization while retaining
the same parent patch length, backbone width and depth, feed-forward width,
training protocol, full test origins, and seeds.  It covers every clean matched
control available in the frozen evaluation records: five data sets and three
seeds.

The atomic encoding wins \planASixteenMatchedWins{} of
\planASixteenMatchedCells{} matched seed--data-set
runs, with an equally weighted pooled H720 change of
\planASixteenMatchedMacroRelative\%.  The mixed per-data-set signs show modest
changes after independent training despite expressive equivalence.  They
motivate examining grouping and width separately in the following grid.

\begin{table}[h]
  \centering
  \caption{Atomic grouping and model-width sensitivity on four representative data
  sets.  Each entry is the three-seed mean $\pm$ sample standard deviation of
  MSE or MAE averaged over $H\in\{96,192,336,720\}$; lower is better and the
  best and second-best distinct means in each data-set/metric group are bold
  and underlined, respectively.  All 36 settings use
  $W=672$, full test origins, fixed P12--16 atoms, and depth two.}
  \label{tab:p12-width}
  \scriptsize
  \setlength{\tabcolsep}{1.8pt}
  \renewcommand{\arraystretch}{0.94}
  \input{generated/table_generated_tangent_p12_width.tex}
\end{table}
\FloatBarrier{}

Across the 12 data-set/$D$ groups, changing $m$ at fixed $D$ gives a mean
within-group normalized range $(\max-\min)/\operatorname{mean}$ of
\pTwelveMRangeMse\% MSE and \pTwelveMRangeMae\% MAE.  Across the 12
data-set/$m$ groups, changing $D$ at fixed $m$ gives
\pTwelveDRangeMse\% and \pTwelveDRangeMae\%, respectively.  The observed range
is smaller along $m$ and larger along $D$, but $D$ also changes backbone
capacity.  Its effect is data-set-specific: ETTm2 and Weather attain their
lowest MSE at $D=64$, while ECL and Traffic do so at $D=256$; width is not
monotonically beneficial.

Because $m=P/12$, changing $m$ also changes the inherited parent span $P$;
this is a parent-scale/atom-count sensitivity audit, not an isolated
atom-count intervention or a test-set model-selection rule.  In the starred
ETTm2 P24/$D=64$ row of \cref{tab:p12-matched}, exact Raw-to-P12--16
transplantation gives zero initial discrepancy, but independent training
changes H720 by $+1.226\%$ (4/12 test seed--block wins).  Preserving the raw
projection class therefore does not guarantee identical results after training.

\FloatBarrier{}

\paragraph{Frozen tangent drift geometry.}
\label{app:tangent-drift-geometry}

For diagnostic analysis only, let
$\Pi_{\mathrm{lev}}=\mathbf{1}\mathbf{1}^{\top}/n$, where
$\mathbf{1}\in\mathbb R^n$ is the all-ones vector on the $n$ analyzed time
coordinates, and decompose each channel's
later-patch displacement into
$d_{\mathrm{lev}}=\Pi_{\mathrm{lev}}(T_\tau-G)$ and
$d_{\mathrm{shape}}=(I-\Pi_{\mathrm{lev}})(T_\tau-G)$.  These coordinates are
orthogonal, sum to $T_\tau-G$, and are never fitted or deployed separately.

We first audit the paper's main local ATD structure.  Every
train-selected $\tau^*$ is reused without reselection, and only the held-forward
fourth training block is measured.  At H720,
\cref{tab:tangent-main-drift-geometry} reports later-patch coordinates,
excluding the protected first patch of each commit.  Level and shape cosines
are computed within their respective orthogonal residual coordinates;
``full--shuffle'' is $\cos(T_{\tau^*}-G,y-G)$ minus the mean cosine of eight
deterministic phase-permuted templates.

\begin{table}[h]
  \centering
  \caption{Frozen held-forward training-block-four H720 geometry for the main local
  ATD structure, averaged over three seeds.  Component signs vary by
  data set; these coordinates are a diagnostic view of the shared direction,
  not separately deployed correction mechanisms.}
  \label{tab:tangent-main-drift-geometry}
  \scriptsize
  \setlength{\tabcolsep}{1.5pt}
  \input{generated/table_generated_tangent_main_geometry.tex}
\end{table}

The complete direction is positively aligned in
\patchGeometryFullPositive/\patchGeometryCells{} seed runs, whereas the level and
shape components are individually positive in
\patchGeometryLevelPositive/\patchGeometryCells{} and
\patchGeometryShapePositive/\patchGeometryCells{} runs.  ECL/Traffic contain
level-negative runs, while ETTh2 is shape-negative in all six width--seed
runs.  This heterogeneity is why the decomposition is used as a view rather
than as two independently sufficient mechanisms.

We then repeat the geometry audit across the three named transfer parents.  For
ATD-2, \cref{tab:tangent-drift-geometry} compares native recursion depths 1 and 3
on later-patch coordinates under the same held-forward protocol.

\begin{table}[h]
  \centering
  \caption{Positive parent--data--seed runs out of 36, with mean explained
  error in parentheses, on the three-parent held-forward grid.  The combined
  direction is consistently stronger than either diagnostic coordinate.}
  \label{tab:tangent-component-risk}
  \footnotesize
  \renewcommand{\arraystretch}{1.08}
  \input{generated/table_generated_tangent_component_risk.tex}
\end{table}

\begin{table}[h]
  \centering
  \caption{Frozen held-forward training-block-four ATD-2 transfer-parent drift geometry, averaged over
  four data sets and three seeds per parent.  Native depths span 192 points for
  AutoTimes and Timer and 256 for TimesFM.}
  \label{tab:tangent-drift-geometry}
  \scriptsize
  \setlength{\tabcolsep}{1.5pt}
  \input{generated/table_generated_tangent_drift_geometry.tex}
\end{table}

The level fraction rises in
\mbox{\driftGeometryLevelIncreases/\driftGeometryDepthTrials} paired summaries, and the
phase-specific cosine margin rises in
\mbox{\driftGeometryMarginIncreases/\driftGeometryDepthTrials}.  Across all cumulative,
band, and depth summaries over later-patch coordinates, the mean-free shape direction has
positive residual alignment in
\mbox{\driftGeometryShapePositiveRows/\driftGeometryFarRows} rows.  At ATD-8 H720, the
true direction exceeds the largest shuffled control in
\mbox{\driftGeometryKEightHSevenTwentyAboveShuffle/\driftGeometryKEightHSevenTwentyTrials}
parent--data--seed runs; the
exception is AutoTimes--ETTh1 seed 2023.  Raw cosine itself rises from depth 1
to 3 in only
\mbox{\driftGeometryCosineIncreases/\driftGeometryDepthTrials} paired summaries, so
this audit supports repeatable drift geometry but does not support a universal
monotone-horizon claim.  These are overlapping descriptive summaries on
held-forward training block 4, not independent significance trials.

\paragraph{Complete tangent selection.}
\label{app:tangent-selection}

The representation-aligned candidate grid is
$\tau\in\{12,24,\ldots,216\}$: multiples of the fixed 12-sample atom, capped at
the largest multiple with at least three occurrences in $W=672$.  The initial
history has times $t=-W,\ldots,-1$ and normalized values
$u_{c,t}=(h_{0,c,t}-\mu_{0,c})/\sigma_{0,c}$.
For phase $r\in\{0,\ldots,\tau-1\}$, let
$I_r=\{t\in\{-W,\ldots,-1\}:t\bmod\tau=r\}$ and
$S_r=I_r\setminus\{\min I_r\}$.
The stored template is $v_{\tau,c,r}=|S_r|^{-1}\sum_{t\in S_r}u_{c,t}$;
it excludes the oldest occurrence of each phase and is never updated.
If $a$ patches have already been committed, local slot $j\in\{1,\ldots,K\}$
and offset $s\in\{0,\ldots,P-1\}$ correspond to
$t=(a+j-1)P+s$ and endpoint $\ell=(a+j)P$.
Using the current rolling history's mean and scale $\mu_{a,c},\sigma_{a,c}$,
the aligned data-space template in \cref{eq:tangent-direction} is
\begin{equation}
  (T_\tau)_{c,j,s}=\mu_{a,c}+\sigma_{a,c}\,v_{\tau,c,t\bmod\tau}.
  \label{eq:template-alignment}
\end{equation}
Thus $T_\tau$ and $G$ both have shape $C\times KP$ after concatenating slots.
The stored phase values stay fixed; only their data-space mapping follows the
current call.  Equivalently, the code subtracts normalized proposals from
$v_\tau$, then multiplies the displacement by $\sigma_a$ before computing
data-space moments or writing back the corrected patch.
The fixed endpoint ramp is the dimensionless scalar
\begin{equation}
 b(\ell)=\begin{cases}
 0.25,&0\leq\ell<96,\\
 0.5,&96\leq\ell<192,\\
 1,&192\leq\ell<336,\\
 2,&336\leq\ell<672,\\
 3,&672\leq\ell.
 \end{cases}
 \label{eq:endpoint-ramp}
\end{equation}
It applies uniformly to all $P$ points and all channels of that patch;
at a boundary the higher band applies.  Local slot $j=1$ is hard-zeroed in
every call, not only at the start of the complete forecast.
These boundaries follow the reporting horizons and lookback rather than a data-set search.
The implementation's fixed factor $\beta=0.0625$ is absorbed into the reported
$\alpha^*=\beta\gamma$; the tables report the effective coefficient, not $\gamma$.

For each data-set--width configuration, 512 uniformly spaced training origins from
each seed are divided chronologically into four contiguous 128-origin blocks.
Moments from seeds 2021--2023 and blocks 1--3 select one period and coefficient
shared by the three seed runs; block 4 confirms the frozen rule without
reselection.  The selector chooses the smallest period attaining the greatest
positive explained error $\max(B_\tau,0)^2/(A_\tau V)$ and then applies
\cref{eq:alpha}.  Thus reported seed runs have distinct trained parents but are
not independent selector refits.

For completeness, for a fixed proposal trajectory and $A_\tau>0$,
\begin{equation}
  R_\tau(\alpha)=\mathbb E\|G+\alpha d_\tau-y\|_2^2
  =R_\tau(0)-2\alpha B_\tau+\alpha^2 A_\tau.
  \label{eq:tangent-risk}
\end{equation}
This strictly convex quadratic has unconstrained minimizer $B_\tau/A_\tau$;
projection onto $\alpha\geq0$ gives \cref{eq:alpha}.  The result is exact for
the fixed training proposal trajectory, while closed-loop correction is tested
empirically because writeback changes later histories.

The ramp has one prescribed relative shape, not five fitted coefficients.  Its
global scale is analytically immaterial: for any $c>0$, replacing $d_\tau$ by
$c d_\tau$ gives $A'_\tau=c^2A_\tau$, $B'_\tau=cB_\tau$, and
$\alpha'^*=\alpha^*/c$.  The explained-risk selector and applied correction
$\alpha'^*c d_\tau=\alpha^*d_\tau$ are therefore unchanged; the five displayed
ramp values do not constitute five fitted coefficients.

\begin{table}[h]
  \centering
  \caption{Train-only Spectrum Tangent sensitivity over seven data sets,
  ATD-4/8, and three seeds.  Training blocks 1--3 reselect and fit each listed
  variant; block 4 confirms the frozen rule.  $R^2$ is equal-configuration
  held-forward explained error, and retained is relative to the fitted staged
  default.  No validation or test example is read.}
  \label{tab:tangent-sensitivity}
  \scriptsize
  \setlength{\tabcolsep}{1.8pt}
  \input{generated/table_generated_tangent_sensitivity.tex}
\end{table}

The linear ramp slightly exceeds the staged ramp on this diagnostic, while
both retain 14/14 configuration and 42/42 seed wins.  Thus the audit supports
increasing long-horizon weight but does not identify the staged breakpoints as
unique.  The flat ramp drops to 13/14 configurations and 38/42 seeds.  Both
alternative period grids retain 14/14 and 42/42, as do coefficient multipliers
from 0.50 through 1.25; at 1.50 the counts fall to 12/14 and 35/42.  These are
raw-trajectory held-forward training risks, not closed-loop test re-ranking.

\Cref{tab:dataset-config} lists the pooled local ATD-4/8 and named-parent
ATD-2/4/8 periods and effective coefficients.
Within each data-set--width configuration, training blocks 1--3 are pooled across the three seeds to
select one shared rule.  A no-pooling control instead refits period and
coefficient from each seed's blocks 1--3 and evaluates its block 4.  It yields
\tangentNoPoolingWins/\tangentNoPoolingTrials{} positive checks; all
\tangentNoPoolingPeriodMatches{} select the pooled $\tau^*$, and coefficients are
\tangentNoPoolingCoefficientRatioMin--\tangentNoPoolingCoefficientRatioMax$\times$
the pooled values.  A complementary leave-one-seed-out control fits on two
seeds and evaluates the excluded seed's block 4.  It likewise yields
\tangentLeaveSeedOutWins/\tangentLeaveSeedOutTrials{} positive checks and all
\tangentLeaveSeedOutPeriodMatches{} pooled periods; its coefficient is
\tangentLeaveSeedOutCoefficientRatioMin--\tangentLeaveSeedOutCoefficientRatioMax$\times$
the pooled value.  Thus three-seed pooling is not required for held-forward
selector stability; both controls are train-only and do not re-estimate test risk.

\paragraph{Weak-seasonality stress test on Exchange-Rate.}
Exchange is deliberately kept outside the seven-data-set mean and used as a
weak-seasonality stress test rather than another positive cell; we do not
assign the selected period a known physical-seasonality interpretation.  We
reuse the same P96/D64/L1 parent interface and three seeds,
train ATD-8 by the same validation-only selector as the other exits,
and then apply the unchanged Spectrum Tangent protocol independently as
Tangent-4 and Tangent-8.  Both widths select $\tau=12$, with
$\alpha^*=\exchangeTangentKFourAlpha{}$ and
$\exchangeTangentKEightAlpha{}$, and all
\exchangeTangentTrainConfirmations/6 held-forward training checks are
positive.

\begin{table*}[t]
  \centering
  \caption{Exchange-Rate stress test over all 798 test origins and three
  seeds.  Entries are mean$\pm$sample standard deviation; positive $\Delta$
  means the correction is worse.  The route is selected only from training
  blocks and is not changed after observing this table.}
  \label{tab:tangent-exchange}
  \scriptsize
  \setlength{\tabcolsep}{2.8pt}
  \input{generated/table_generated_tangent_exchange.tex}
\end{table*}

The preserved first proposal slot makes H96 identical, but the correction
worsens H720 MSE/MAE by
\exchangeTangentKFourHSevenTwentyMseRelative\%/
\exchangeTangentKFourHSevenTwentyMaeRelative\% for Tangent-4 and
\exchangeTangentKEightHSevenTwentyMseRelative\%/
\exchangeTangentKEightHSevenTwentyMaeRelative\% for Tangent-8.  It wins 0/6 H720
seed runs and only \exchangeTangentHSevenTwentyBlockWins/24 reporting blocks.
The gap between positive held-forward training geometry and negative test
risk is therefore retained as a distribution-shift counterexample: the
closed-form coefficient is exact for its fixed training trajectory, not a
guarantee that a history template transports across regimes.

\paragraph{Horizon and H720 configuration detail.}
\label{app:tangent-detail}

\begin{figure}[h]
  \centering
  \includegraphics[width=0.94\columnwidth]{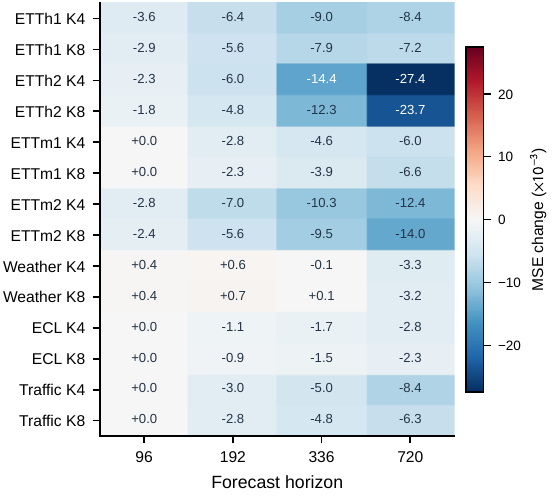}
  \caption{Per-configuration MSE change from \atd{} to Spectrum Tangent.  The
  correction is strongest and most consistent at H720; short-horizon effects
  are smaller and mixed.}
  \label{fig:tangent-boundary}
\end{figure}

\begin{table*}[t]
  \centering
  \caption{Detailed local H720 accounting.  Every row averages all test
  origins and three seeds.  Seed runs and test seed--blocks count
  strict improvements over the matched \atd{} model; negative
  $\Delta$ is better.}
  \label{tab:tangent-h720-internal}
  \scriptsize
  \setlength{\tabcolsep}{4.0pt}
  \input{generated/table_generated_tangent_h720_internal.tex}
\end{table*}

%% file: generated/table_generated_tangent_external_full.tex
\begin{tabular*}{\textwidth}{@{\extracolsep{\fill}}l*{18}{r}@{}}
\toprule
Models & \multicolumn{2}{c}{Timer-XL} & \multicolumn{2}{c}{Timer} & \multicolumn{2}{c}{UniTST} & \multicolumn{2}{c}{iTrans.} & \multicolumn{2}{c}{DLinear} & \multicolumn{2}{c}{PatchTST} & \multicolumn{2}{c}{TimesNet} & \multicolumn{2}{c}{Nonstat.} & \multicolumn{2}{c}{Autoformer} \\
 & \multicolumn{2}{c}{\scriptsize (\citeyear{liu2025timerxl})} & \multicolumn{2}{c}{\scriptsize (\citeyear{liu2024timer})} & \multicolumn{2}{c}{\scriptsize (\citeyear{liu2024unitst})} & \multicolumn{2}{c}{\scriptsize (\citeyear{liu2024itransformer})} & \multicolumn{2}{c}{\scriptsize (\citeyear{zeng2023dlinear})} & \multicolumn{2}{c}{\scriptsize (\citeyear{nie2023patchtst})} & \multicolumn{2}{c}{\scriptsize (\citeyear{wu2023timesnet})} & \multicolumn{2}{c}{\scriptsize (\citeyear{liu2022nonstationary})} & \multicolumn{2}{c}{\scriptsize (\citeyear{wu2021autoformer})} \\
Metric & MSE & MAE & MSE & MAE & MSE & MAE & MSE & MAE & MSE & MAE & MSE & MAE & MSE & MAE & MSE & MAE & MSE & MAE \\
\midrule
\multicolumn{19}{l}{\textit{ETTh1}} \\
96 & \textbf{0.364} & \textbf{0.397} & 0.371 & 0.404 & 0.379 & 0.415 & 0.387 & 0.418 & \underline{0.369} & \underline{0.400} & 0.373 & 0.403 & 0.452 & 0.463 & 0.452 & 0.478 & 0.467 & 0.499 \\
192 & \textbf{0.405} & \underline{0.424} & \underline{0.407} & 0.429 & 0.415 & 0.438 & 0.416 & 0.437 & \textbf{0.405} & \textbf{0.422} & \textbf{0.405} & 0.425 & 0.474 & 0.477 & 0.484 & 0.510 & 0.492 & 0.523 \\
336 & \underline{0.427} & \textbf{0.439} & 0.434 & 0.445 & 0.440 & 0.454 & 0.434 & 0.450 & 0.435 & 0.445 & \textbf{0.423} & \underline{0.440} & 0.493 & 0.489 & 0.511 & 0.522 & 0.519 & 0.531 \\
720 & \textbf{0.439} & \textbf{0.459} & 0.461 & \underline{0.466} & 0.482 & 0.482 & 0.447 & 0.473 & 0.493 & 0.508 & \underline{0.445} & 0.471 & 0.560 & 0.534 & 0.571 & 0.543 & 0.589 & 0.560 \\
\addlinespace[1pt]
Avg. & \textbf{0.409} & \textbf{0.430} & 0.418 & 0.436 & 0.429 & 0.447 & 0.421 & 0.445 & 0.426 & 0.444 & \underline{0.412} & \underline{0.435} & 0.495 & 0.491 & 0.505 & 0.513 & 0.517 & 0.528 \\
\midrule
\multicolumn{19}{l}{\textit{ETTh2}} \\
96 & \textbf{0.277} & \textbf{0.343} & \underline{0.285} & \underline{0.344} & 0.343 & 0.398 & 0.304 & 0.362 & 0.305 & 0.371 & 0.289 & 0.347 & 0.340 & 0.374 & 0.348 & 0.403 & 0.358 & 0.397 \\
192 & \textbf{0.348} & \textbf{0.391} & 0.365 & 0.400 & 0.376 & 0.420 & 0.372 & 0.407 & 0.412 & 0.439 & \underline{0.360} & \underline{0.393} & 0.402 & 0.414 & 0.408 & 0.448 & 0.435 & 0.451 \\
336 & \textbf{0.375} & \textbf{0.418} & 0.412 & 0.440 & 0.399 & 0.435 & 0.418 & 0.440 & 0.527 & 0.508 & \underline{0.389} & \underline{0.420} & 0.452 & 0.452 & 0.424 & 0.457 & 0.454 & 0.475 \\
720 & \underline{0.409} & 0.458 & 0.468 & 0.487 & 0.419 & \underline{0.457} & 0.463 & 0.476 & 0.830 & 0.653 & \textbf{0.398} & \textbf{0.440} & 0.462 & 0.468 & 0.448 & 0.476 & 0.479 & 0.492 \\
\addlinespace[1pt]
Avg. & \textbf{0.352} & \underline{0.402} & 0.382 & 0.418 & 0.384 & 0.428 & 0.389 & 0.421 & 0.518 & 0.493 & \underline{0.359} & \textbf{0.400} & 0.414 & 0.427 & 0.407 & 0.446 & 0.431 & 0.454 \\
\midrule
\multicolumn{19}{l}{\textit{ETTm1}} \\
96 & 0.290 & \underline{0.341} & \textbf{0.281} & \textbf{0.338} & 0.289 & 0.348 & 0.311 & 0.365 & 0.307 & 0.350 & \underline{0.285} & 0.346 & 0.338 & 0.375 & 0.414 & 0.414 & 0.466 & 0.466 \\
192 & 0.337 & \underline{0.369} & \underline{0.330} & \textbf{0.368} & 0.332 & 0.375 & 0.353 & 0.390 & 0.337 & \textbf{0.368} & \textbf{0.329} & 0.372 & 0.371 & 0.387 & 0.524 & 0.482 & 0.504 & 0.496 \\
336 & 0.374 & \underline{0.392} & 0.367 & 0.393 & \underline{0.365} & 0.397 & 0.387 & 0.411 & 0.366 & \textbf{0.387} & \textbf{0.363} & 0.394 & 0.410 & 0.411 & 0.541 & 0.497 & 0.574 & 0.530 \\
720 & 0.437 & 0.428 & 0.432 & 0.433 & \underline{0.421} & 0.431 & 0.452 & 0.445 & \textbf{0.419} & \textbf{0.419} & \underline{0.421} & \underline{0.426} & 0.478 & 0.450 & 0.578 & 0.509 & 0.596 & 0.558 \\
\addlinespace[1pt]
Avg. & 0.359 & \underline{0.382} & \underline{0.352} & 0.383 & \underline{0.352} & 0.388 & 0.376 & 0.403 & 0.357 & \textbf{0.381} & \textbf{0.349} & 0.385 & 0.399 & 0.406 & 0.514 & 0.475 & 0.535 & 0.512 \\
\midrule
\multicolumn{19}{l}{\textit{ETTm2}} \\
96 & 0.175 & \textbf{0.257} & 0.175 & \textbf{0.257} & \underline{0.171} & 0.260 & 0.183 & 0.272 & \textbf{0.167} & 0.263 & 0.172 & \underline{0.259} & 0.187 & 0.267 & 0.237 & 0.306 & 0.255 & 0.339 \\
192 & 0.242 & 0.301 & 0.239 & 0.301 & \textbf{0.228} & \textbf{0.230} & 0.250 & 0.315 & \underline{0.230} & 0.311 & 0.233 & \underline{0.299} & 0.249 & 0.309 & 0.330 & 0.387 & 0.279 & 0.335 \\
336 & 0.293 & 0.337 & 0.293 & 0.342 & \underline{0.282} & \underline{0.336} & 0.311 & 0.356 & 0.298 & 0.361 & \textbf{0.280} & \textbf{0.331} & 0.321 & 0.351 & 0.404 & 0.424 & 0.331 & 0.374 \\
720 & \underline{0.376} & \underline{0.390} & 0.392 & 0.407 & 0.380 & 0.398 & 0.417 & 0.419 & 0.432 & 0.446 & \textbf{0.357} & \textbf{0.382} & 0.497 & 0.403 & 0.525 & 0.486 & 0.413 & 0.450 \\
\addlinespace[1pt]
Avg. & 0.271 & 0.322 & 0.275 & 0.327 & \underline{0.265} & \textbf{0.306} & 0.290 & 0.340 & 0.282 & 0.345 & \textbf{0.261} & \underline{0.318} & 0.314 & 0.333 & 0.374 & 0.401 & 0.320 & 0.374 \\
\midrule
\multicolumn{19}{l}{\textit{ECL}} \\
96 & \textbf{0.127} & \textbf{0.219} & \underline{0.129} & \underline{0.221} & 0.130 & 0.225 & 0.133 & 0.229 & 0.138 & 0.238 & 0.132 & 0.232 & 0.184 & 0.288 & 0.185 & 0.287 & 0.256 & 0.357 \\
192 & \textbf{0.145} & \textbf{0.236} & \underline{0.148} & \underline{0.239} & 0.150 & 0.244 & 0.158 & 0.258 & 0.152 & 0.251 & 0.151 & 0.250 & 0.192 & 0.295 & 0.282 & 0.368 & 0.291 & 0.376 \\
336 & \textbf{0.159} & \textbf{0.252} & \underline{0.164} & \underline{0.256} & 0.166 & 0.262 & 0.168 & 0.262 & 0.167 & 0.268 & 0.171 & 0.272 & 0.200 & 0.303 & 0.289 & 0.377 & 0.290 & 0.379 \\
720 & \textbf{0.187} & \textbf{0.277} & \underline{0.201} & \underline{0.289} & 0.206 & 0.297 & 0.205 & 0.294 & 0.203 & 0.302 & 0.222 & 0.318 & 0.228 & 0.325 & 0.305 & 0.399 & 0.320 & 0.403 \\
\addlinespace[1pt]
Avg. & \textbf{0.155} & \textbf{0.246} & \underline{0.161} & \underline{0.251} & 0.163 & 0.257 & 0.164 & 0.258 & 0.165 & 0.265 & 0.169 & 0.268 & 0.201 & 0.303 & 0.265 & 0.358 & 0.289 & 0.379 \\
\midrule
\multicolumn{19}{l}{\textit{Traffic}} \\
96 & \textbf{0.340} & \textbf{0.238} & \underline{0.348} & \underline{0.240} & 0.359 & 0.250 & 0.353 & 0.259 & 0.399 & 0.285 & 0.359 & 0.255 & 0.593 & 0.315 & 0.610 & 0.322 & 0.675 & 0.412 \\
192 & \textbf{0.360} & \textbf{0.247} & \underline{0.369} & \underline{0.250} & 0.373 & 0.257 & 0.373 & 0.267 & 0.409 & 0.290 & 0.377 & 0.265 & 0.596 & 0.317 & 0.626 & 0.346 & 0.679 & 0.423 \\
336 & \textbf{0.377} & \textbf{0.256} & 0.388 & \underline{0.260} & \underline{0.386} & 0.265 & \underline{0.386} & 0.275 & 0.422 & 0.297 & 0.393 & 0.276 & 0.600 & 0.319 & 0.633 & 0.352 & 0.688 & 0.440 \\
720 & \textbf{0.418} & \textbf{0.279} & 0.431 & \underline{0.285} & \underline{0.421} & 0.286 & 0.425 & 0.296 & 0.461 & 0.319 & 0.436 & 0.305 & 0.619 & 0.335 & 0.651 & 0.366 & 0.693 & 0.457 \\
\addlinespace[1pt]
Avg. & \textbf{0.374} & \textbf{0.255} & \underline{0.384} & \underline{0.259} & 0.385 & 0.265 & \underline{0.384} & 0.274 & 0.423 & 0.298 & 0.391 & 0.275 & 0.602 & 0.322 & 0.630 & 0.347 & 0.684 & 0.433 \\
\midrule
\multicolumn{19}{l}{\textit{Weather}} \\
96 & 0.157 & \underline{0.205} & \underline{0.151} & \textbf{0.202} & 0.152 & 0.206 & 0.174 & 0.225 & 0.169 & 0.229 & \textbf{0.149} & \textbf{0.202} & 0.169 & 0.228 & 0.185 & 0.241 & 0.355 & 0.409 \\
192 & 0.206 & 0.250 & \underline{0.196} & \textbf{0.245} & 0.198 & \underline{0.249} & 0.227 & 0.268 & 0.211 & 0.268 & \textbf{0.194} & \textbf{0.245} & 0.222 & 0.269 & 0.286 & 0.325 & 0.421 & 0.450 \\
336 & 0.259 & 0.291 & \underline{0.249} & \underline{0.288} & 0.251 & 0.291 & 0.290 & 0.309 & 0.258 & 0.306 & \textbf{0.244} & \textbf{0.285} & 0.290 & 0.310 & 0.323 & 0.347 & 0.452 & 0.465 \\
720 & 0.337 & 0.344 & 0.330 & 0.344 & 0.322 & \underline{0.340} & 0.374 & 0.360 & \underline{0.320} & 0.362 & \textbf{0.317} & \textbf{0.338} & 0.376 & 0.364 & 0.436 & 0.401 & 0.513 & 0.496 \\
\addlinespace[1pt]
Avg. & 0.240 & 0.273 & 0.232 & \underline{0.270} & \underline{0.231} & 0.272 & 0.266 & 0.291 & 0.239 & 0.291 & \textbf{0.226} & \textbf{0.268} & 0.264 & 0.293 & 0.308 & 0.329 & 0.435 & 0.455 \\
\bottomrule
\end{tabular*}

%% file: generated/table_generated_tangent_local_full.tex
\begin{tabular*}{\textwidth}{@{\extracolsep{\fill}}l*{14}{r}@{}}
\toprule
Models & \multicolumn{2}{c}{Parent} & \multicolumn{2}{c}{Joint-4} & \multicolumn{2}{c}{ATD-4} & \multicolumn{2}{c}{+Tangent-4} & \multicolumn{2}{c}{Joint-8} & \multicolumn{2}{c}{ATD-8} & \multicolumn{2}{c}{+Tangent-8} \\
Horizon & MSE & MAE & MSE & MAE & MSE & MAE & MSE & MAE & MSE & MAE & MSE & MAE & MSE & MAE \\
\midrule
\multicolumn{15}{l}{\textit{ETTh1}} \\
96 & 0.326 & 0.375 & 0.349 & 0.386 & 0.325 & 0.374 & \textbf{0.321} & \textbf{0.370} & 0.360 & 0.393 & 0.325 & 0.374 & \underline{0.322} & \underline{0.371} \\
192 & 0.350 & 0.393 & 0.385 & 0.409 & 0.348 & 0.392 & \textbf{0.342} & \textbf{0.386} & 0.400 & 0.417 & 0.350 & 0.392 & \underline{0.345} & \underline{0.387} \\
336 & 0.363 & 0.408 & 0.414 & 0.430 & 0.361 & 0.406 & \textbf{0.352} & \textbf{0.397} & 0.436 & 0.443 & 0.362 & 0.405 & \underline{0.354} & \underline{0.398} \\
720 & 0.408 & 0.443 & 0.487 & 0.483 & 0.405 & 0.440 & \textbf{0.397} & \textbf{0.433} & 0.522 & 0.502 & 0.405 & 0.439 & \underline{0.398} & \underline{0.434} \\
\addlinespace[1pt]
Avg. & 0.362 & 0.405 & 0.409 & 0.427 & 0.360 & 0.403 & \textbf{0.353} & \textbf{0.396} & 0.430 & 0.439 & 0.361 & 0.403 & \underline{0.355} & \underline{0.398} \\
\midrule
\multicolumn{15}{l}{\textit{ETTh2}} \\
96 & \underline{0.235} & \underline{0.312} & 0.296 & 0.364 & \underline{0.235} & \underline{0.312} & \textbf{0.233} & \textbf{0.311} & 0.296 & 0.372 & \underline{0.235} & \underline{0.312} & \textbf{0.233} & \underline{0.312} \\
192 & 0.279 & \underline{0.348} & 0.321 & 0.382 & 0.282 & 0.351 & \textbf{0.276} & \textbf{0.347} & 0.321 & 0.390 & 0.282 & 0.350 & \underline{0.277} & \underline{0.348} \\
336 & \underline{0.313} & \underline{0.380} & 0.348 & 0.402 & 0.322 & 0.386 & \textbf{0.308} & \textbf{0.376} & 0.343 & 0.407 & 0.326 & 0.388 & 0.314 & \underline{0.380} \\
720 & 0.402 & 0.445 & 0.463 & 0.465 & 0.419 & 0.455 & \textbf{0.392} & \textbf{0.435} & 0.433 & 0.461 & 0.421 & 0.455 & \underline{0.397} & \underline{0.439} \\
\addlinespace[1pt]
Avg. & 0.307 & 0.371 & 0.357 & 0.403 & 0.314 & 0.376 & \textbf{0.302} & \textbf{0.367} & 0.348 & 0.408 & 0.316 & 0.376 & \underline{0.305} & \underline{0.370} \\
\midrule
\multicolumn{15}{l}{\textit{ETTm1}} \\
96 & \textbf{0.272} & \textbf{0.339} & \underline{0.297} & \underline{0.360} & \textbf{0.272} & \textbf{0.339} & \textbf{0.272} & \textbf{0.339} & 0.319 & 0.376 & \textbf{0.272} & \textbf{0.339} & \textbf{0.272} & \textbf{0.339} \\
192 & 0.321 & 0.368 & 0.326 & 0.378 & 0.320 & 0.367 & \underline{0.317} & \textbf{0.364} & 0.342 & 0.391 & 0.319 & \underline{0.366} & \textbf{0.316} & \textbf{0.364} \\
336 & 0.363 & 0.391 & 0.360 & 0.396 & 0.360 & 0.390 & \textbf{0.355} & \textbf{0.384} & 0.371 & 0.406 & \underline{0.359} & \underline{0.389} & \textbf{0.355} & \textbf{0.384} \\
720 & 0.426 & 0.426 & 0.417 & 0.428 & 0.421 & 0.424 & \underline{0.415} & \underline{0.415} & 0.422 & 0.432 & 0.419 & 0.422 & \textbf{0.412} & \textbf{0.413} \\
\addlinespace[1pt]
Avg. & 0.346 & 0.381 & 0.350 & 0.391 & 0.343 & 0.380 & \underline{0.340} & \textbf{0.375} & 0.363 & 0.401 & 0.342 & \underline{0.379} & \textbf{0.339} & \textbf{0.375} \\
\midrule
\multicolumn{15}{l}{\textit{ETTm2}} \\
96 & \underline{0.182} & \textbf{0.263} & \textbf{0.181} & 0.265 & 0.184 & 0.266 & \textbf{0.181} & \underline{0.264} & 0.190 & 0.270 & 0.184 & 0.266 & \textbf{0.181} & \underline{0.264} \\
192 & 0.245 & 0.306 & \underline{0.242} & 0.306 & 0.244 & 0.306 & \textbf{0.237} & \textbf{0.302} & 0.252 & 0.311 & 0.243 & \underline{0.305} & \textbf{0.237} & \textbf{0.302} \\
336 & 0.300 & 0.343 & 0.293 & 0.342 & 0.297 & 0.342 & \underline{0.287} & \underline{0.334} & 0.303 & 0.346 & 0.296 & 0.340 & \textbf{0.286} & \textbf{0.333} \\
720 & 0.399 & 0.409 & \underline{0.374} & 0.398 & 0.388 & 0.401 & 0.376 & \underline{0.387} & 0.381 & 0.400 & 0.387 & 0.399 & \textbf{0.373} & \textbf{0.386} \\
\addlinespace[1pt]
Avg. & 0.282 & 0.330 & 0.273 & 0.328 & 0.278 & 0.329 & \underline{0.270} & \underline{0.322} & 0.281 & 0.332 & 0.277 & 0.328 & \textbf{0.269} & \textbf{0.321} \\
\midrule
\multicolumn{15}{l}{\textit{ECL}} \\
96 & \textbf{0.116} & \textbf{0.216} & \underline{0.129} & \underline{0.234} & \textbf{0.116} & \textbf{0.216} & \textbf{0.116} & \textbf{0.216} & 0.137 & 0.245 & \textbf{0.116} & \textbf{0.216} & \textbf{0.116} & \textbf{0.216} \\
192 & \textbf{0.136} & \underline{0.234} & 0.146 & 0.248 & \underline{0.137} & \underline{0.234} & \textbf{0.136} & \textbf{0.233} & 0.153 & 0.258 & \underline{0.137} & \underline{0.234} & \textbf{0.136} & \textbf{0.233} \\
336 & 0.157 & \underline{0.253} & 0.165 & 0.265 & 0.157 & \underline{0.253} & \textbf{0.155} & \textbf{0.252} & 0.172 & 0.274 & 0.157 & \underline{0.253} & \underline{0.156} & \textbf{0.252} \\
720 & \underline{0.201} & \underline{0.289} & 0.209 & 0.301 & 0.202 & 0.290 & \textbf{0.199} & \textbf{0.287} & 0.213 & 0.305 & 0.202 & 0.290 & \textbf{0.199} & \textbf{0.287} \\
\addlinespace[1pt]
Avg. & \underline{0.153} & \underline{0.248} & 0.162 & 0.262 & \underline{0.153} & \underline{0.248} & \textbf{0.152} & \textbf{0.247} & 0.169 & 0.271 & \underline{0.153} & \underline{0.248} & \textbf{0.152} & \textbf{0.247} \\
\midrule
\multicolumn{15}{l}{\textit{Traffic}} \\
96 & \textbf{0.348} & \textbf{0.242} & \underline{0.353} & \underline{0.247} & \textbf{0.348} & \textbf{0.242} & \textbf{0.348} & \textbf{0.242} & 0.355 & 0.249 & \textbf{0.348} & \textbf{0.242} & \textbf{0.348} & \textbf{0.242} \\
192 & \textbf{0.366} & \textbf{0.252} & \underline{0.367} & \underline{0.254} & 0.369 & \underline{0.254} & \textbf{0.366} & \textbf{0.252} & 0.369 & 0.255 & 0.369 & 0.255 & \textbf{0.366} & \textbf{0.252} \\
336 & 0.386 & 0.265 & 0.385 & \underline{0.263} & 0.388 & 0.267 & \textbf{0.383} & \textbf{0.262} & 0.385 & 0.265 & 0.388 & 0.267 & \underline{0.384} & \underline{0.263} \\
720 & 0.436 & 0.293 & \underline{0.431} & 0.289 & 0.438 & 0.295 & \textbf{0.430} & \textbf{0.286} & \underline{0.431} & 0.288 & 0.437 & 0.295 & \underline{0.431} & \underline{0.287} \\
\addlinespace[1pt]
Avg. & \underline{0.384} & 0.263 & \underline{0.384} & 0.263 & 0.386 & 0.265 & \textbf{0.382} & \textbf{0.260} & 0.385 & 0.264 & 0.386 & 0.265 & \textbf{0.382} & \underline{0.261} \\
\midrule
\multicolumn{15}{l}{\textit{Weather}} \\
96 & \textbf{0.154} & \textbf{0.204} & 0.173 & 0.238 & \underline{0.156} & \underline{0.206} & \underline{0.156} & \underline{0.206} & 0.208 & 0.277 & \underline{0.156} & \underline{0.206} & \underline{0.156} & \underline{0.206} \\
192 & \underline{0.201} & \underline{0.250} & 0.215 & 0.276 & \textbf{0.200} & \textbf{0.249} & \underline{0.201} & \underline{0.250} & 0.246 & 0.307 & \underline{0.201} & \underline{0.250} & \underline{0.201} & \underline{0.250} \\
336 & \underline{0.254} & \underline{0.290} & 0.269 & 0.316 & \textbf{0.252} & \underline{0.290} & \textbf{0.252} & \textbf{0.289} & 0.292 & 0.339 & \textbf{0.252} & \underline{0.290} & \textbf{0.252} & \underline{0.290} \\
720 & 0.330 & \underline{0.342} & 0.344 & 0.364 & 0.328 & \underline{0.342} & \textbf{0.324} & \textbf{0.339} & 0.380 & 0.393 & 0.328 & \underline{0.342} & \underline{0.325} & \textbf{0.339} \\
\addlinespace[1pt]
Avg. & 0.235 & \underline{0.272} & 0.250 & 0.299 & \underline{0.234} & \underline{0.272} & \textbf{0.233} & \textbf{0.271} & 0.281 & 0.329 & \underline{0.234} & \underline{0.272} & \underline{0.234} & \textbf{0.271} \\
\bottomrule
\end{tabular*}

%% file: generated/table_generated_tangent_joint_stability.tex
\begin{tabular*}{\textwidth}{@{\extracolsep{\fill}}ll*{4}{r}@{}}
\toprule
Method & Metric & ETTh1 & ETTh2 & ETTm1 & ETTm2 \\
\midrule
Joint Direct-4 & MSE & $0.4090\!\pm\!0.0162$ & $0.3571\!\pm\!0.0108$ & $0.3499\!\pm\!0.0002$ & $0.2728\!\pm\!0.0082$ \\
 & MAE & $0.4270\!\pm\!0.0097$ & $0.4032\!\pm\!0.0015$ & $0.3906\!\pm\!0.0019$ & $0.3277\!\pm\!0.0045$ \\
\addlinespace[1pt]
Joint Direct-8 & MSE & $0.4295\!\pm\!0.0291$ & $0.3483\!\pm\!0.0124$ & $0.3635\!\pm\!0.0015$ & $0.2812\!\pm\!0.0079$ \\
 & MAE & $0.4389\!\pm\!0.0123$ & $0.4076\!\pm\!0.0117$ & $0.4010\!\pm\!0.0015$ & $0.3319\!\pm\!0.0048$ \\
\addlinespace[1pt]
ATD-4 & MSE & $0.3598\!\pm\!0.0013$ & $0.3144\!\pm\!0.0056$ & $0.3432\!\pm\!0.0001$ & $0.2785\!\pm\!0.0006$ \\
 & MAE & $0.4029\!\pm\!0.0008$ & $0.3760\!\pm\!0.0036$ & $0.3799\!\pm\!0.0003$ & $0.3288\!\pm\!0.0022$ \\
\addlinespace[1pt]
ATD-8 & MSE & $0.3606\!\pm\!0.0016$ & $0.3161\!\pm\!0.0048$ & $0.3421\!\pm\!0.0002$ & $0.2773\!\pm\!0.0012$ \\
 & MAE & $0.4028\!\pm\!0.0013$ & $0.3765\!\pm\!0.0028$ & $0.3790\!\pm\!0.0005$ & $0.3277\!\pm\!0.0026$ \\
\bottomrule
\end{tabular*}
\par\medskip
\begin{tabular*}{\textwidth}{@{\extracolsep{\fill}}ll*{4}{r}@{}}
\toprule
Method & Metric & ECL & Traffic & Weather & Avg. \\
\midrule
Joint Direct-4 & MSE & $0.1622\!\pm\!0.0021$ & $0.3840\!\pm\!0.0050$ & $0.2501\!\pm\!0.0062$ & $0.3121\!\pm\!0.0069$ \\
 & MAE & $0.2621\!\pm\!0.0010$ & $0.2635\!\pm\!0.0011$ & $0.2987\!\pm\!0.0025$ & $0.3390\!\pm\!0.0032$ \\
\addlinespace[1pt]
Joint Direct-8 & MSE & $0.1687\!\pm\!0.0010$ & $0.3849\!\pm\!0.0073$ & $0.2814\!\pm\!0.0277$ & $0.3225\!\pm\!0.0124$ \\
 & MAE & $0.2706\!\pm\!0.0012$ & $0.2643\!\pm\!0.0026$ & $0.3289\!\pm\!0.0270$ & $0.3490\!\pm\!0.0087$ \\
\addlinespace[1pt]
ATD-4 & MSE & $0.1530\!\pm\!0.0002$ & $0.3859\!\pm\!0.0014$ & $0.2340\!\pm\!0.0023$ & $0.2955\!\pm\!0.0016$ \\
 & MAE & $0.2483\!\pm\!0.0002$ & $0.2645\!\pm\!0.0009$ & $0.2719\!\pm\!0.0023$ & $0.3246\!\pm\!0.0015$ \\
\addlinespace[1pt]
ATD-8 & MSE & $0.1529\!\pm\!0.0003$ & $0.3856\!\pm\!0.0013$ & $0.2342\!\pm\!0.0023$ & $0.2955\!\pm\!0.0017$ \\
 & MAE & $0.2482\!\pm\!0.0004$ & $0.2646\!\pm\!0.0008$ & $0.2719\!\pm\!0.0023$ & $0.3244\!\pm\!0.0015$ \\
\bottomrule
\end{tabular*}

%% file: generated/table_generated_tangent_dataset_config.tex
\begin{tabular*}{\linewidth}{@{\extracolsep{\fill}}lrrrrrr@{}}
\toprule
Dataset & $P$ & $D$ & $L$ & $\tau^*$ & $\alpha^*_{4}$ & $\alpha^*_{8}$ \\
\midrule
ETTh1 & 24 & 64 & 1 & 24 & 0.4815 & 0.3997 \\
ETTh2 & 48 & 64 & 1 & 24 & 0.2828 & 0.2141 \\
ETTm1 & 96 & 64 & 1 & 96 & 0.1827 & 0.1592 \\
ETTm2 & 24 & 64 & 2 & 96 & 0.3977 & 0.3459 \\
Weather & 48 & 64 & 2 & 12 & 0.0851 & 0.0856 \\
ECL & 96 & 256 & 2 & 168 & 0.1133 & 0.0907 \\
Traffic & 96 & 256 & 2 & 168 & 0.1235 & 0.1114 \\
\bottomrule
\end{tabular*}

%% file: generated/table_generated_tangent_named_config.tex
\begin{tabular*}{\linewidth}{@{\extracolsep{\fill}}llrrrrrr@{}}
\toprule
Parent & Dataset & $\tau^*_2$ & $\alpha^*_2$ & $\tau^*_4$ & $\alpha^*_4$ & $\tau^*_8$ & $\alpha^*_8$ \\
\midrule
AutoTimes & ETTh1 & 24 & 0.1799 & 24 & 0.1152 & 36 & 0.0424 \\
AutoTimes & ETTh2 & 24 & 0.5033 & 24 & 0.3760 & 24 & 0.2347 \\
AutoTimes & ETTm1 & 96 & 0.2341 & 96 & 0.2051 & 96 & 0.1903 \\
AutoTimes & Weather & 144 & 0.3473 & 144 & 0.3104 & 144 & 0.2777 \\
\midrule
Timer & ETTh1 & 24 & 0.1626 & 24 & 0.1502 & 24 & 0.1596 \\
Timer & ETTh2 & 24 & 0.2413 & 24 & 0.2423 & 24 & 0.2599 \\
Timer & ETTm1 & 96 & 0.3327 & 96 & 0.3050 & 96 & 0.3081 \\
Timer & Weather & 144 & 0.3037 & 144 & 0.2799 & 144 & 0.2839 \\
\midrule
TimesFM & ETTh1 & 24 & 0.1917 & 24 & 0.1818 & 24 & 0.1935 \\
TimesFM & ETTh2 & 24 & 0.2389 & 24 & 0.2236 & 24 & 0.2302 \\
TimesFM & ETTm1 & 96 & 0.2931 & 96 & 0.3018 & 96 & 0.3123 \\
TimesFM & Weather & 144 & 0.2993 & 144 & 0.3015 & 144 & 0.3107 \\
\bottomrule
\end{tabular*}

%% file: generated/table_generated_tangent_target_fit.tex
\begin{tabular*}{\textwidth}{@{\extracolsep{\fill}}lrrrrr@{}}
\toprule
& Energy & \multicolumn{2}{c}{Unexplained} & \multicolumn{2}{c}{Parent easier} \\
\cmidrule(lr){3-4}\cmidrule(l){5-6}
Parent / split & Parent/future & Parent & Future & Runs & Seed--blocks \\
\midrule
Local parent / train & 0.150$\times$ & \textbf{0.109} & 0.845 & 42/42 & 168/168 \\
TimesFM / val. & 0.139$\times$ & \textbf{0.126} & 0.861 & 12/12 & 48/48 \\
\bottomrule
\end{tabular*}

%% file: generated/table_generated_tangent_target_control_raw.tex
\begin{tabular*}{\linewidth}{@{\extracolsep{\fill}}lrrrrrr@{}}
\toprule
& \multicolumn{2}{c}{Forecast MSE} & \multicolumn{2}{c}{Parent-traj. MSE} & \multicolumn{2}{c}{Wins} \\
Data & Direct & ATD & Direct & ATD & Forecast & Fidel. \\
\midrule
ETTh1 & 0.4410 & 0.4052 & 0.0661 & 0.0098 & 3/3 & 3/3 \\
ETTh2 & 2.3903 & 0.4208 & 1.7784 & 0.0104 & 3/3 & 3/3 \\
ETTm1 & 0.4261 & 0.4186 & 0.0488 & 0.0145 & 3/3 & 3/3 \\
ETTm2 & 0.3736 & 0.3865 & 0.0613 & 0.0332 & 0/3 & 3/3 \\
Weather & 0.3356 & 0.3278 & 0.0385 & 0.0116 & 3/3 & 3/3 \\
ECL & 0.2063 & 0.2017 & 0.0187 & 0.0038 & 3/3 & 3/3 \\
Traffic & 0.4328 & 0.4368 & 0.0305 & 0.0094 & 0/3 & 3/3 \\
\bottomrule
\end{tabular*}

%% file: generated/table_generated_tangent_stride.tex
\begin{tabular*}{\linewidth}{@{\extracolsep{\fill}}llclc@{}}
\toprule
Method & Reported/local scope & Target online? & Accuracy report & Speedup \\
\midrule
STRIDE$^\dagger$ & Timer-XL, ECL, H720, B1, width 3 & verify & MSE 0.200$\to$0.240 & 1.82$\times$ \\
STRIDE$^\dagger$ & Timer-XL, Weather, H720, B1, width 3 & verify & MSE 0.320$\to$0.362 & 1.42$\times$ \\
\midrule
\textbf{ATD-8} & local H720, B1 & no & 7-data MSE 0.3720$\to$0.3711 & \textbf{5.54$\times$} \\
\textbf{+Tangent} & local H720, B1 & no & 7-data MSE 0.3711$\to$0.3620 & 3.24$\times$ \\
\bottomrule
\end{tabular*}

%% file: generated/table_generated_tangent_named_parents.tex
\begin{tabular*}{\linewidth}{@{\extracolsep{\fill}}llrrrrrr@{}}
\toprule
& & & ATD/parent & \multicolumn{2}{c}{Tangent} & \multicolumn{2}{c}{ATD/Direct} \\
\cmidrule(lr){4-4}\cmidrule(lr){5-6}\cmidrule(lr){7-8}
Parent & Data & ATD-1 exact & Forecast wins & $\Delta$ (\%) & Wins & Forecast wins & Fidel. wins \\
\midrule
AutoTimes & ETTh1 & 3/3 & 0/3 & -1.56 & 3/3 & 3/3 & 3/3 \\
 & ETTh2 & 3/3 & 2/3 & -6.14 & 3/3 & 3/3 & 3/3 \\
 & ETTm1 & 3/3 & 3/3 & -0.84 & 3/3 & 3/3 & 3/3 \\
 & Weather & 3/3 & 3/3 & -3.07 & 3/3 & 3/3 & 3/3 \\
\midrule
Timer & ETTh1 & 3/3 & 3/3 & -1.93 & 3/3 & 3/3 & 3/3 \\
 & ETTh2 & 3/3 & 0/3 & -4.19 & 3/3 & 3/3 & 3/3 \\
 & ETTm1 & 3/3 & 3/3 & -15.66 & 3/3 & 0/3 & 3/3 \\
 & Weather & 3/3 & 3/3 & -5.83 & 3/3 & 0/3 & 3/3 \\
\midrule
TimesFM & ETTh1 & 3/3 & 3/3 & -2.04 & 3/3 & 3/3 & 3/3 \\
 & ETTh2 & 3/3 & 3/3 & -6.84 & 3/3 & 3/3 & 3/3 \\
 & ETTm1 & 3/3 & 3/3 & -5.06 & 3/3 & 3/3 & 3/3 \\
 & Weather & 3/3 & 3/3 & -12.32 & 3/3 & 0/3 & 3/3 \\
\bottomrule
\end{tabular*}

%% file: generated/table_generated_tangent_transfer_stability.tex
\begin{tabular*}{\linewidth}{@{\extracolsep{\fill}}lrrrr@{}}
\toprule
ATD width & AutoTimes & Timer & TimesFM & \textbf{Total} \\
\midrule
ATD-2 & 38/48 & 33/48 & 48/48 & \textbf{119/144} \\
ATD-4 & 40/48 & 33/48 & 48/48 & \textbf{121/144} \\
ATD-8 & 44/48 & 45/48 & 48/48 & \textbf{137/144} \\
\bottomrule
\end{tabular*}

%% file: generated/table_generated_tangent_plan_a16.tex
\begin{tabular*}{\linewidth}{@{\extracolsep{\fill}}lrrrrr@{}}
\toprule
Dataset & \multicolumn{3}{c}{P12--16 interface} & Parent & Full rank \\
\cmidrule(lr){2-4}\cmidrule(lr){5-5}\cmidrule(lr){6-6}
& $P$ & $m=P/12$ & concat. $16m$ & $D$ & $\operatorname{rank}(\phi)=12$ \\
\midrule
ETTh1 & 24 & 2 & 32 & 64 & 3/3 \\
ETTh2 & 48 & 4 & 64 & 64 & 3/3 \\
ETTm1 & 96 & 8 & 128 & 64 & 3/3 \\
ETTm2 & 24 & 2 & 32 & 64 & 3/3 \\
Weather & 48 & 4 & 64 & 64 & 3/3 \\
ECL & 96 & 8 & 128 & 256 & 3/3 \\
Traffic & 96 & 8 & 128 & 256 & 3/3 \\
\bottomrule
\end{tabular*}

%% file: generated/table_generated_tangent_plan_a16_matched.tex
\begin{tabular*}{\linewidth}{@{\extracolsep{\fill}}lrrllr@{}}
\toprule
Dataset & \multicolumn{2}{c}{Architecture} & \multicolumn{2}{c}{H720 MSE} & Paired $\Delta$ \\
\cmidrule(lr){2-3}\cmidrule(lr){4-5}\cmidrule(lr){6-6}
& $P$ & $D$ & Raw & P12--16 & (\%) \\
\midrule
ETTh1 & 24 & 64 & $\mathbf{0.4079}\!\pm\!0.0037$ & $0.4084\!\pm\!0.0012$ & +0.15 \\
ETTh2 & 48 & 64 & $0.4174\!\pm\!0.0138$ & $\mathbf{0.4025}\!\pm\!0.0087$ & -3.46 \\
ETTm1 & 96 & 64 & $0.4292\!\pm\!0.0012$ & $\mathbf{0.4264}\!\pm\!0.0011$ & -0.64 \\
ETTm2$^\ast$ & 24 & 64 & $\mathbf{0.3858}\!\pm\!0.0029$ & $0.3905\!\pm\!0.0031$ & +1.23 \\
Weather & 48 & 64 & $\mathbf{0.3278}\!\pm\!0.0050$ & $0.3301\!\pm\!0.0050$ & +0.72 \\
\bottomrule
\end{tabular*}

%% file: generated/table_generated_tangent_p12_width.tex
\begin{tabular*}{\linewidth}{@{\extracolsep{\fill}}lrlrrr@{}}
\toprule
Dataset & Atoms & Metric & \multicolumn{3}{c}{Model width $D$} \\
\cmidrule(lr){4-6}
& $m$ & & 64 & 128 & 256 \\
\midrule
ETTm2 & 2 & MSE & $0.2816\!\pm\!0.0014$ & $0.2769\!\pm\!0.0058$ & $0.2999\!\pm\!0.0429$ \\
 &  & MAE & $0.3302\!\pm\!0.0020$ & $0.3250\!\pm\!0.0015$ & $0.3367\!\pm\!0.0221$ \\
 & 4 & MSE & $\mathbf{0.2764}\!\pm\!0.0039$ & $\underline{0.2765}\!\pm\!0.0031$ & $0.2854\!\pm\!0.0026$ \\
 &  & MAE & $\mathbf{0.3231}\!\pm\!0.0029$ & $\underline{0.3242}\!\pm\!0.0021$ & $0.3290\!\pm\!0.0014$ \\
 & 8 & MSE & $0.2835\!\pm\!0.0034$ & $0.2844\!\pm\!0.0031$ & $0.2862\!\pm\!0.0065$ \\
 &  & MAE & $0.3257\!\pm\!0.0012$ & $0.3276\!\pm\!0.0025$ & $0.3295\!\pm\!0.0046$ \\
\midrule
Weather & 2 & MSE & $\mathbf{0.2333}\!\pm\!0.0007$ & $0.2508\!\pm\!0.0248$ & $0.2408\!\pm\!0.0062$ \\
 &  & MAE & $\underline{0.2708}\!\pm\!0.0006$ & $0.2842\!\pm\!0.0219$ & $0.2750\!\pm\!0.0079$ \\
 & 4 & MSE & $0.2348\!\pm\!0.0032$ & $0.2373\!\pm\!0.0043$ & $0.2497\!\pm\!0.0127$ \\
 &  & MAE & $0.2718\!\pm\!0.0029$ & $\mathbf{0.2698}\!\pm\!0.0022$ & $0.2806\!\pm\!0.0092$ \\
 & 8 & MSE & $\underline{0.2343}\!\pm\!0.0003$ & $0.2401\!\pm\!0.0042$ & $0.2365\!\pm\!0.0021$ \\
 &  & MAE & $0.2723\!\pm\!0.0005$ & $0.2749\!\pm\!0.0019$ & $0.2739\!\pm\!0.0028$ \\
\midrule
ECL & 2 & MSE & $0.1609\!\pm\!0.0003$ & $0.1556\!\pm\!0.0006$ & $0.1572\!\pm\!0.0007$ \\
 &  & MAE & $0.2569\!\pm\!0.0003$ & $0.2504\!\pm\!0.0007$ & $0.2499\!\pm\!0.0004$ \\
 & 4 & MSE & $0.1632\!\pm\!0.0006$ & $0.1568\!\pm\!0.0005$ & $\underline{0.1538}\!\pm\!0.0001$ \\
 &  & MAE & $0.2593\!\pm\!0.0004$ & $0.2509\!\pm\!0.0003$ & $\mathbf{0.2470}\!\pm\!0.0002$ \\
 & 8 & MSE & $0.1631\!\pm\!0.0005$ & $0.1570\!\pm\!0.0002$ & $\mathbf{0.1526}\!\pm\!0.0002$ \\
 &  & MAE & $0.2596\!\pm\!0.0003$ & $0.2525\!\pm\!0.0000$ & $\underline{0.2476}\!\pm\!0.0003$ \\
\midrule
Traffic & 2 & MSE & $0.4006\!\pm\!0.0012$ & $0.3938\!\pm\!0.0006$ & $0.3921\!\pm\!0.0016$ \\
 &  & MAE & $0.2756\!\pm\!0.0016$ & $0.2678\!\pm\!0.0009$ & $0.2655\!\pm\!0.0006$ \\
 & 4 & MSE & $0.4113\!\pm\!0.0011$ & $0.3919\!\pm\!0.0014$ & $\mathbf{0.3828}\!\pm\!0.0009$ \\
 &  & MAE & $0.2858\!\pm\!0.0006$ & $0.2703\!\pm\!0.0005$ & $\mathbf{0.2628}\!\pm\!0.0012$ \\
 & 8 & MSE & $0.4138\!\pm\!0.0010$ & $0.3971\!\pm\!0.0009$ & $\underline{0.3841}\!\pm\!0.0013$ \\
 &  & MAE & $0.2879\!\pm\!0.0006$ & $0.2739\!\pm\!0.0006$ & $\underline{0.2629}\!\pm\!0.0009$ \\
\bottomrule
\end{tabular*}

%% file: generated/table_generated_tangent_main_geometry.tex
\begin{tabular*}{\linewidth}{@{\extracolsep{\fill}}lrrrrr@{}}
\toprule
Data/width & Level frac. & Level cos. & Shape cos. & Full cos. & Full--shuffle \\
\midrule
ETTh1 ATD-4 & 0.265 & 0.407 & 0.167 & 0.243 & 0.246 \\
ETTh1 ATD-8 & 0.254 & 0.317 & 0.171 & 0.214 & 0.224 \\
\cmidrule(lr){1-6}
ETTh2 ATD-4 & 0.502 & 0.128 & -0.027 & 0.056 & 0.013 \\
ETTh2 ATD-8 & 0.506 & 0.137 & -0.031 & 0.068 & 0.013 \\
\cmidrule(lr){1-6}
ETTm1 ATD-4 & 0.277 & 0.318 & 0.140 & 0.193 & 0.158 \\
ETTm1 ATD-8 & 0.273 & 0.258 & 0.166 & 0.195 & 0.173 \\
\cmidrule(lr){1-6}
ETTm2 ATD-4 & 0.533 & 0.241 & 0.001 & 0.124 & 0.020 \\
ETTm2 ATD-8 & 0.524 & 0.196 & 0.009 & 0.107 & 0.019 \\
\cmidrule(lr){1-6}
Weather ATD-4 & 0.270 & 0.314 & 0.099 & 0.118 & 0.000 \\
Weather ATD-8 & 0.267 & 0.311 & 0.094 & 0.118 & 0.000 \\
\cmidrule(lr){1-6}
ECL ATD-4 & 0.287 & -0.051 & 0.130 & 0.089 & 0.037 \\
ECL ATD-8 & 0.294 & -0.021 & 0.127 & 0.081 & 0.030 \\
\cmidrule(lr){1-6}
Traffic ATD-4 & 0.063 & -0.030 & 0.109 & 0.102 & 0.051 \\
Traffic ATD-8 & 0.064 & -0.071 & 0.101 & 0.093 & 0.047 \\
\midrule
\textbf{Macro ATD-4} & 0.314 & 0.189 & 0.088 & 0.132 & 0.075 \\
\textbf{Macro ATD-8} & 0.312 & 0.161 & 0.091 & 0.125 & 0.072 \\
\bottomrule
\end{tabular*}

%% file: generated/table_generated_tangent_component_risk.tex
\begin{tabular*}{\linewidth}{@{\extracolsep{\fill}}lrrr@{}}
\toprule
Direction & ATD-2, depth 1 & ATD-2, depth 3 & ATD-8, depth 1 \\
\midrule
Level & 33/36 (0.026) & 36/36 (0.053) & 33/36 (0.036) \\
Shape & 36/36 (0.052) & 36/36 (0.070) & 36/36 (0.058) \\
\textbf{Combined} & \textbf{36/36 (0.070)} & \textbf{36/36 (0.118)} & \textbf{36/36 (0.089)} \\
\bottomrule
\end{tabular*}

%% file: generated/table_generated_tangent_drift_geometry.tex
\begin{tabular*}{\linewidth}{@{\extracolsep{\fill}}lrrrrr@{}}
\toprule
Parent/depth & Level frac. & Level cos. & Shape cos. & Full cos. & Full--shuffle \\
\midrule
AutoTimes, depth 1 & 0.428 & 0.061 & 0.263 & 0.169 & 0.091 \\
AutoTimes, depth 3 & 0.571 & 0.193 & 0.274 & 0.242 & 0.117 \\
\cmidrule(lr){1-6}
Timer, depth 1 & 0.418 & 0.208 & 0.360 & 0.290 & 0.248 \\
Timer, depth 3 & 0.542 & 0.261 & 0.400 & 0.351 & 0.294 \\
\cmidrule(lr){1-6}
TimesFM, depth 1 & 0.445 & 0.339 & 0.221 & 0.275 & 0.159 \\
TimesFM, depth 3 & 0.634 & 0.307 & 0.317 & 0.334 & 0.240 \\
\midrule
\textbf{Macro, depth 1} & 0.430 & 0.203 & 0.281 & 0.245 & 0.166 \\
\textbf{Macro, depth 3} & 0.582 & 0.254 & 0.330 & 0.309 & 0.217 \\
\bottomrule
\end{tabular*}

%% file: generated/table_generated_tangent_sensitivity.tex
\begin{tabular*}{\linewidth}{@{\extracolsep{\fill}}llrrrr@{}}
\toprule
Factor & Choice & Cfg. wins & Seed wins & $R^2$ (\%) & Retained (\%) \\
\midrule
Ramp & \textbf{staged (used)} & 14/14 & 42/42 & 1.772 & 100.0 \\
 & flat & 13/14 & 38/42 & 1.307 & 73.8 \\
 & linear & 14/14 & 42/42 & 1.806 & 101.9 \\
\midrule
Period grid & \textbf{12:12:216 (used)} & 14/14 & 42/42 & 1.772 & 100.0 \\
 & 24:24:216 & 14/14 & 42/42 & 1.772 & 100.0 \\
 & 12:12:168 & 14/14 & 42/42 & 1.772 & 100.0 \\
\midrule
Coefficient & 0.50 & 14/14 & 42/42 & 1.298 & 73.2 \\
 & 0.75 & 14/14 & 42/42 & 1.638 & 92.4 \\
 & \textbf{1.00 (fit)} & 14/14 & 42/42 & 1.772 & 100.0 \\
 & 1.25 & 14/14 & 42/42 & 1.701 & 96.0 \\
 & 1.50 & 12/14 & 35/42 & 1.423 & 80.3 \\
\bottomrule
\end{tabular*}

%% file: generated/table_generated_tangent_exchange.tex
\begin{tabular*}{\linewidth}{@{\extracolsep{\fill}}llrrrrrr@{}}
\toprule
&& \multicolumn{3}{c}{MSE} & \multicolumn{3}{c}{MAE} \\
\cmidrule(lr){3-5}\cmidrule(lr){6-8}
Mode & H & ATD & +Tangent & $\Delta\%$ & ATD & +Tangent & $\Delta\%$ \\
\midrule
ATD-4 & 96 & $0.1109\!\pm\!0.0078$ & $0.1109\!\pm\!0.0078$ & +0.00 & $0.2269\!\pm\!0.0078$ & $0.2269\!\pm\!0.0078$ & +0.00 \\
ATD-4 & 192 & \textbf{$0.2199\!\pm\!0.0062$} & $0.2282\!\pm\!0.0069$ & +3.80 & \textbf{$0.3285\!\pm\!0.0071$} & $0.3344\!\pm\!0.0075$ & +1.80 \\
ATD-4 & 336 & \textbf{$0.4026\!\pm\!0.0111$} & $0.4205\!\pm\!0.0124$ & +4.45 & \textbf{$0.4573\!\pm\!0.0096$} & $0.4670\!\pm\!0.0094$ & +2.12 \\
ATD-4 & 720 & \textbf{$1.0264\!\pm\!0.0247$} & $1.0554\!\pm\!0.0266$ & +2.82 & \textbf{$0.7610\!\pm\!0.0101$} & $0.7704\!\pm\!0.0092$ & +1.22 \\
\addlinespace[1pt]
ATD-8 & 96 & $0.1109\!\pm\!0.0078$ & $0.1109\!\pm\!0.0078$ & +0.00 & $0.2269\!\pm\!0.0078$ & $0.2269\!\pm\!0.0078$ & +0.00 \\
ATD-8 & 192 & \textbf{$0.2175\!\pm\!0.0079$} & $0.2258\!\pm\!0.0088$ & +3.80 & \textbf{$0.3273\!\pm\!0.0082$} & $0.3331\!\pm\!0.0087$ & +1.76 \\
ATD-8 & 336 & \textbf{$0.3939\!\pm\!0.0055$} & $0.4121\!\pm\!0.0075$ & +4.62 & \textbf{$0.4527\!\pm\!0.0069$} & $0.4627\!\pm\!0.0073$ & +2.20 \\
ATD-8 & 720 & \textbf{$0.9870\!\pm\!0.0163$} & $1.0291\!\pm\!0.0144$ & +4.26 & \textbf{$0.7475\!\pm\!0.0032$} & $0.7610\!\pm\!0.0032$ & +1.81 \\
\bottomrule
\end{tabular*}

%% file: generated/table_generated_tangent_h720_internal.tex
\begin{tabular*}{\linewidth}{@{\extracolsep{\fill}}lrrrrrr@{}}
\toprule
Dataset & Mode & ATD & +Tangent & $\Delta$ & Runs & Test blocks \\
\midrule
ETTh1 & ATD-4 & 0.4054 & 0.3970 & -0.0084 & 3/3 & 11/12 \\
ETTh1 & ATD-8 & 0.4052 & 0.3980 & -0.0072 & 3/3 & 11/12 \\
\cmidrule(lr){1-7}
ETTh2 & ATD-4 & 0.4191 & 0.3917 & -0.0274 & 3/3 & 11/12 \\
ETTh2 & ATD-8 & 0.4208 & 0.3972 & -0.0237 & 3/3 & 11/12 \\
\cmidrule(lr){1-7}
ETTm1 & ATD-4 & 0.4210 & 0.4151 & -0.0060 & 3/3 & 9/12 \\
ETTm1 & ATD-8 & 0.4186 & 0.4120 & -0.0066 & 3/3 & 9/12 \\
\cmidrule(lr){1-7}
ETTm2 & ATD-4 & 0.3884 & 0.3760 & -0.0124 & 3/3 & 9/12 \\
ETTm2 & ATD-8 & 0.3865 & 0.3726 & -0.0140 & 3/3 & 9/12 \\
\cmidrule(lr){1-7}
Weather & ATD-4 & 0.3278 & 0.3245 & -0.0033 & 3/3 & 11/12 \\
Weather & ATD-8 & 0.3278 & 0.3246 & -0.0032 & 3/3 & 12/12 \\
\cmidrule(lr){1-7}
ECL & ATD-4 & 0.2018 & 0.1990 & -0.0028 & 3/3 & 9/12 \\
ECL & ATD-8 & 0.2017 & 0.1994 & -0.0023 & 3/3 & 8/12 \\
\cmidrule(lr){1-7}
Traffic & ATD-4 & 0.4384 & 0.4301 & -0.0084 & 3/3 & 12/12 \\
Traffic & ATD-8 & 0.4368 & 0.4305 & -0.0063 & 3/3 & 11/12 \\
\bottomrule
\end{tabular*}